\documentclass{article}

\usepackage[utf8]{inputenc}
\usepackage{amsmath,amsfonts,amsthm}
\usepackage{tikz}
\usetikzlibrary{external}
\usetikzlibrary{intersections}
\usepackage{hyperref}
\usepackage{stmaryrd}
\usepackage{comment}
\usepackage{caption}
\usepackage{subcaption}
\usepackage{pgfplots}
\usepackage{cite}

\DeclareMathOperator{\sign}{sign}
\DeclareMathOperator{\Lip}{Lip}
\DeclareMathOperator{\Div}{div}

\DeclareMathOperator{\proj}{proj}
\DeclareMathOperator{\id}{id}
\DeclareMathOperator{\TV}{TV}
\DeclareMathOperator{\TGV}{TGV}
\DeclareMathOperator{\KR}{KR}
\DeclareMathOperator{\BD}{BD}

\DeclareMathOperator{\support}{supp}
\DeclareMathOperator{\diam}{diam}
\DeclareMathOperator{\dom}{dom}
\newcommand{\calE}{\mathcal{E}}
\newcommand{\grad}{\nabla}
\newcommand{\RR}{\mathbb{R}}
\newcommand{\ind}[1]{I_{#1}}
\newcommand{\norm}[2][]{\|#2\|_{#1}}
\newcommand{\abs}[1]{|#1|}
\newcommand{\scp}[2]{\langle #1, #2\rangle}
\newcommand{\set}[2]{\{#1\ :\ #2\}}
\newcommand{\sett}[1]{\{#1\}}
\newcommand{\radon}{\mathfrak{M}}
\newcommand{\push}{\#}
\newcommand{\mathvar}[1]{\textup{#1}}

\newcommand{\BVspace}{\mathvar{BV}}

\newcommand{\defeq}{:=}
\newcommand{\field}[1]{\mathbb{#1}}
\newcommand{\R}{\field{R}}
\renewcommand{\d}{\,\mathrm{d}} %\text{d}}
\newcommand{\bdry}{\partial}
\newcommand{\downto}{\searrow}

\newcommand{\Union}{\bigcup}

\newcommand{\union}{\cup}
\newcommand{\isect}{\cap}
\newcommand{\closure}[1]{\overline #1}
\newcommand{\Lebesgue}{\mathfrak{L}}
\newcommand{\Hausdorff}{\mathfrak{H}}
\newcommand{\restrict}{\llcorner}

\makeatletter
\def \weaktostar@sym{\setbox0=\hbox{$\rightharpoonup$}\rlap{\hbox 
        to\wd0{\hss\raise1ex\hbox{$\scriptscriptstyle{*\,}$}\hss}}\box0}
    \def \weaktostar    {\mathrel{\weaktostar@sym}}
\makeatother

\def\Xint#1{\mathchoice
{\XXint\displaystyle\textstyle{#1}}%
{\XXint\textstyle\scriptstyle{#1}}%
{\XXint\scriptstyle\scriptscriptstyle{#1}}%
{\XXint\scriptscriptstyle\scriptscriptstyle{#1}}%
\!\int}
\def\XXint#1#2#3{{\setbox0=\hbox{$#1{#2#3}{\int}$ }
\vcenter{\hbox{$#2#3$ }}\kern-.6\wd0}}

\def\dashint{\Xint-}

\makeatletter
\newcommand*{\centerfloat}{%
  \parindent \z@
  \leftskip \z@ \@plus 1fil \@minus \textwidth
  \rightskip\leftskip
  \parfillskip \z@skip}
\makeatother

\newtheorem{lemma}{Lemma}[section]
\newtheorem{theorem}[lemma]{Theorem}
\newtheorem{corollary}[lemma]{Corollary}

\theoremstyle{definition}

\newtheorem{remark}[lemma]{Remark}

\title{Imaging with Kantorovich-Rubinstein discrepancy}

\author{Jan Lellmann\thanks{Department for Applied Mathematics and Theoretical Physics, University of Cambridge, United Kingdom, \texttt{j.lellmann@damtp.cam.ac.uk}}
    \and
    Dirk A. Lorenz\thanks{Institute for Analysis and Algebra, TU Braunschweig, 38092 Braunschweig, Germany, \texttt{d.lorenz@tu-braunschweig.de}}
    \and
    Carola Schönlieb\thanks{Department for Applied Mathematics and Theoretical Physics, University of Cambridge, United Kingdom}
    \and 
    Tuomo Valkonen\thanks{
        Prometeo Fellow, Center for Mathematical Modeling (Modemat), EPN Quito, Ecuador
        }
}

\usepackage[UKenglish]{babel}
\begin{document}

\maketitle
\begin{abstract}
   We propose the use of the Kantorovich-Rubinstein norm from optimal
  transport in imaging problems. In particular, we discuss a variational regularisation model endowed with a Kantorovich-Rubinstein
  discrepancy term and total variation regularization in the context of image denoising and
  cartoon-texture decomposition. We point out
  connections of this approach to several other recently proposed methods such as total
  generalized variation and norms capturing oscillating patterns. We also
  show that the respective optimization problem can be turned into a
  convex-concave saddle point problem with simple constraints and
  hence, can be solved by standard tools. Numerical examples exhibit
  interesting features and favourable performance for denoising and
  cartoon-texture decomposition.
\end{abstract}

\section{Introduction}
\label{sec:intro}

Distance functions related to ideas from optimal transport have
appeared in various places in imaging problems in the last ten years. The main applications in this context are image and shape classification \cite{memoli2007use,memoli2008gromov,memoli2009spectral,memoli2011spectral,memoli2011gromov,oudre2012classification,rabin2010geodesic,swoboda2013convex}, segmentation \cite{chan2007histogram,ni2009local,peyre2012wasserstein,schmitzer2013modelling,schmitzer2013object}, registration and warping \cite{haker2004optimal,zhu2007imagemorphing,papadakis2014optimal}, image smoothing \cite{burger2012densityestimation}, contrast and colour modification \cite{rabin2011wasserstein,ferradans2013regularized}, texture synthesis and texture mixing \cite{rabin2012wasserstein}, and surface mapping \cite{lipman2011conformal,1281.65034,boyer2011algorithms,bunn2011comparing}. Being a distance function applicable to very general densities (continuous and discrete (Dirac deltas) densities)
the Wasserstein distance had an increasing impact on robust distance measures in imaging \cite{rubner2000earth,buttazzo2005optimalplanning,grauman2004fast,ling2007efficient,wang2011optimal,rabin2012wasserstein,peyre2012wasserstein,burger2012densityestimation}.
In most cases, the 2-Wasserstein
distance \cite{ambrosio2006gradient} is used.

In this work we propose the use of the so-called
Kantorovich-Rubinstein norm ($\KR$-norm) in imaging. We investigate the $\KR$-$\TV$ denoising problem, that is, for a given noisy image $u^0$ on a set $\Omega$ and two constants $\lambda_1,\lambda_2\geq 0$ we consider
\[
\min_u \norm[\KR,(\lambda_1,\lambda_2)]{u - u^0} + \TV(u)
\]
where the $\KR$-norm is defined for a Radon measure $\mu$ (and hence,
also for $L^1$-functions) on a set $\Omega\subset\RR^n$ by
\[
\norm[\KR,(\lambda_1,\lambda_2)]{\mu} = \sup\set{\int_\Omega f\d\mu}{\abs{f}\leq\lambda_1,\ \Lip(f)\leq\lambda_2}.
\]
The Kantorovich-Rubinstein norm~\cite[§8.3]{bogachev2007measure} is closely related to the
1-Wasserstein distance and hence, to optimal transport problems.  It
will turn out that this norm has interesting relations to other well
known concepts in imaging: The KR-norm is a generalization of the
$L^1$ norm, and hence, a $\KR$-$\TV$ denoising model inherits and
generalizes some of the favorable properties of the $L^1$-TV
denoising~\cite{chan2005aspects}.  The generalization of $L^1$-norm
discrepancies to $\KR$-norm discrepancies shares some similarities
with the generalization from the TV penalty to the total generalized
variation (TGV) penalty~\cite{bredies2010tgv}.  Finally, the KR-norm
discrepancy shares properties with Meyer's $G$-norm
model~\cite{meyer2002oscillating,vese2003modelingtextures} for
oscillating patterns and for cartoon-texture decomposition.  Also from
the computational point of view, the $\KR$-norm has favorable
properties. It turns out that the $\KR$-$\TV$ denoising problem has a
formulation as a saddle-point problem that can be solved by means of
several primal-dual methods. The computational cost per iteration as
well as the needed storage requirements are almost as low as for
similar algorithms for $L^1$-$\TV$ denoising.

The paper is organized as follows: After fixing the notation we
introduce and recall transport metrics in
Section~\ref{sec:transport-metrics}.  In Section~\ref{sec:primalKR} we
derive two reformulations of the $\KR$-norm that will be used to
analyze and interpret the $\KR$-$\TV$ denoising problem, which is the
content of Section~\ref{sec:KR-TV-denoising}. In
Section~\ref{sec:numerical-solution} we illustrate how the $\KR$-$\TV$
denoising problem can be solved numerically by primal dual methods.
Finally, in Section~\ref{sec:experiments} we present examples for
$\KR$-$\TV$ denoising and cartoon-texture decomposition and then finish the paper with a conclusion.

\subsection{Notation}
\label{sec:notation}

We work in a domain $\Omega\subset\RR^n$ and use $\abs{x}$ as the
euclidean absolute value for $x\in\Omega$. We denote by
$\radon(\Omega,\RR^n)$ the space of $\RR^n$-valued Radon measures,
i.e. the dual space of
$(C_0(\Omega,\RR^n),\norm[\infty]{\abs{\cdot}})$ of continuous
functions that vanish ``at infinity''. If we want to emphasize that a
function or a measure is vector valued we write $\vec{\nu}$ but
sometime we omit the emphasis. The dual pairing between
$\radon(\Omega,\RR^n)$ and $C_0(\Omega,\RR^n)$ (and any two other
spaces in duality) will be denoted by $\scp{\vec{f}}{\vec{\mu}}$.
Consequently, the norm on $\radon(\Omega,\RR^n)$ is
$\norm[\radon]{\vec{\mu}} = \sup_{\abs{\vec{f}}\leq 1}\int
\vec{f}\cdot\d \vec{\mu}$ and is called the Radon norm. We identify
$u\in L^1(\Omega,\RR^n)$ with the corresponding measure
$u\in\radon(\Omega,\RR^n)$, i.e. we treat $L^1(\Omega,\RR^n)$ embedded
into $\radon(\Omega,\RR^n)$.  The $n$-dimensional Lebesgue measure is
denoted by $\Lebesgue^n$ while the $d$-dimensional Hausdorff measure
is $\Hausdorff^d$.

For a measure $\mu$ on $\Omega$, another set $\Omega'$ and
$F:\Omega\to\Omega'$ the push-forward of $\mu$ by $F$ is $\mu\push
F(A) = \mu(F^{-1}(A))$. On $\Omega\times \Omega$ we denote by
$\proj_{1/2}$ the projections onto the first and second component,
respectively. Having a measure $\gamma$ on $\Omega\times\Omega$ we
denote (with slight abuse of notation) by $\proj_{1/2}\gamma$ the push
forward of $\gamma$ by $\proj_{1/2}$, i.e. the marginals of
$\gamma$. The restriction of some measure $\mu$ onto some set $A$ is
denoted by $\mu\restrict A$.  By $C_b(\Omega,\RR^n)$ we denote the
space of bounded and continuous functions on $\Omega$. For
$f:\Omega\to\RR$ we denote by $\Lip(f) = \sup_{x\neq y}
\abs{f(x)-f(y)}/\abs{x-y}$ the Lipschitz constant of $f$.

For two points $a,b\in\RR^n$ we define the line interval $[a, b] = \{
t a + (1-t) b \mid t \in [0, 1]\}$ and the vector measure $\llbracket
a, b \rrbracket$ to be
\[
\llbracket a,b\rrbracket = \frac{b-a}{\abs{b-a}}\Hausdorff^1\restrict [a,b].
\]

By $\diam(\Omega) = \sup\set{\abs{x-y}}{x,y\in\Omega}$ we denote the
diameter on $\Omega$. For a set $C$ we denote by $\ind{C}$ the
indicator function, i.e. $\ind{C}(u) = 0$ for $u\in C$ and $=\infty$
otherwise.

\section{Transport metrics}
\label{sec:transport-metrics}

A variety of different metrics exist on measure spaces.
As the study of metrics on measure spaces has its origins in probability
theory, most metrics are defined on the space of probability measures,
i.e., non-negative measures with total mass equal to one. A popular
class of such metrics is given by the Wasserstein metrics: For $p\geq
1$ and two probability measures $\mu$ and $\nu$ define
\begin{equation}
W_p(\mu,\nu) = \Big( \inf\set{\int_{\Omega\times\Omega} \abs{x-y}^p\d
  \gamma(x,y)}{\proj_1\gamma = \mu,\ \proj_2\gamma = \nu}\Big)^{1/p}.\label{eq:p-wasserstein}
\end{equation}
Note that this metric also makes sense if $\mu$ and $\nu$ are not
probability measures but still non-negative and have equal
mass,
i.e., $\int_\Omega\d\mu = \int_\Omega \d\nu$.  However, if the mass is not equal,
no $\gamma$ with $\mu$ and $\nu$ as marginals
would exist.

The celebrated Kantorovich duality~\cite{kantorovich1942translocation,villani2009optimaltransport} states that, in the case of non-negative measures with equal mass, the
Wasserstein metric can be equivalently expressed as 
\[
W_p(\mu,\nu) = \Big(\sup\set{\int_\Omega \phi\d\mu +
  \int_\Omega\psi\d\nu}{\phi,\psi\in C_b(\Omega),\ \phi(x) + \psi(y)
  \leq |x-y|^p}\Big)^{1/p}.%\footnote{Usually the duality is stated for probability measures. However, for two non-negative measures with equal mass, the duality obviously still holds true. If two non-negative measure do not have the same mass, \eqref{eq:p-wasserstein} gives $+\infty$ since the infimum is taken over the empty set, while the dual formulation gives $+\infty$ since we can make the value of the sum of the integrals arbitrarily large by adding and subtracting a constant from the functions $\phi$ and $\psi$.}
\]
A particular special case is $p=1$, and here, the Kantorovich-Rubinstein
duality~\cite{kantorovich1957rubinstein,villani2009optimaltransport}
states that
\[
W_1(\mu,\nu) = \sup\set{\int_\Omega f\d(\mu-\nu)}{\Lip(f)\leq 1}.
\]
A particularly interesting fact is that this metric only
depends on the difference $\mu-\nu$. In fact, by setting
\[
\norm[\Lip^*]{\mu} = \sup\set{\int_\Omega f\d\mu}{\Lip(f)\leq 1}
\]
one obtains the so-called dual Lipschitz norm on the space of
measures with zero mean and finite first moments (cf.~\cite[§8.10(viii)]{bogachev2007measure} where it is called modified Kantorovich-Rubinstein norm). 
Note that the supremum is
unbounded if one has a nonzero mean. To prevent the
norm from blowing up in this case, 
and hence, to
obtain a norm on the space of all signed measures with finite first
moments, one can add the constraint that the test functions $f$ shall
be bounded. This leads to the expression
\[
\sup\set{\int_\Omega f\d\mu}{\abs{f}\leq 1,\ \Lip(f)\leq 1}.
\]
(which is called Kantorovich-Rubinstein norm in \cite[§8.3]{bogachev2007measure}).
Since we would like the bound on the values of $f$ and the bound on its Lipschitz
constant to vary independently in the following, we introduce for
$\lambda = (\lambda_1,\lambda_2)$ the norm
\begin{equation}
  \norm[\KR,\lambda]{\mu} =\sup\set{\int_\Omega
    f\d\mu}{\abs{f}\leq\lambda_1,\ \Lip(f)\leq \lambda_2}\label{eq:KR-def}
\end{equation}
Note that in the extreme cases $\lambda_1=\infty$ and $\lambda_2=\infty$  we recover the dual Lipschitz and the Radon norm
\begin{equation}
  \begin{split}
    \norm[\KR,(\infty,1)]{\mu} &= \norm[\Lip^*]{\mu}\\
    \norm[\KR,(1,\infty)]{\mu} &= \norm[\radon]{\mu}.
  \end{split}\label{eq:kr-specia-case}
\end{equation}
Note that the norm $\norm[\KR,(\lambda_1,\lambda_2)]{\mu}$ with $\lambda_1,\lambda_2>0$ is equivalent to the bounded Lipschitz norm~\cite[§6]{villani2009optimaltransport} where one takes the supremum over all functions $f$ such that $\abs{f} + \Lip(f)\leq 1$.
In general we have the following simple estimates:
\begin{lemma}[Estimates by the Radon norm]
  \label{lem:KRnorm-estimates}
  For any $\lambda = (\lambda_1,\lambda_2)\geq 0$ it holds that
  \[
  \norm[\KR,\lambda]{\mu}\leq \lambda_1\norm[\radon]{\mu}.
  \]
  If $\mu$ is non-negative it holds that
  \[
  \norm[\KR,\lambda]{\mu}= \lambda_1\norm[\radon]{\mu}.  
  \]
  If $\Omega$ has finite diameter $\diam(\Omega)$, then it
  holds for any $\mu$ with $\int_\Omega \d \mu = 0$ that
  \[
  \norm[\KR,\lambda]{\mu}\leq
  \lambda_2\tfrac{\diam(\Omega)}{2}\norm[\radon]{\mu}.
  \]
\end{lemma}
\begin{proof}
  The first inequality follows directly from the definition of $\norm[\KR,\lambda]{\mu}$
  by dropping the constraint $\abs{\grad
    f}\leq\lambda_2$ and the second claim by observing that the
  supremum is attained at $f\equiv\lambda_1$.
  
  For the last claim we estimate from above by dropping the constraint
  $\norm[\infty]{f}\leq\lambda_1$. However, since $\Omega$ has bounded
  diameter and $\mu$ has mean value zero, the constraint
  $\norm[\infty]{\abs{\grad f}}\leq \lambda_2$ implies that one also
  has a bound $\norm[\infty]{f}\leq\lambda_2\diam(\Omega)/2$ (indeed,
  $\lambda_2\diam(\Omega)$ is a bound on the value $\max f - \min f$,
  however, since $\int_\Omega \d \mu = 0$, we may add a constant to $f$
  without altering the outer supremum). We obtain
  \[
  \norm[\KR,\lambda_1,\lambda_2]{\mu} \leq
  \sup\limits_{\norm[\infty]{f}\leq\lambda_2\diam(\Omega)/2}\int f\,\d\mu
  \leq \lambda_2\diam(\Omega)\norm[\radon]{\mu}/2.
  \]
\end{proof}

\begin{remark}
  Note that the $\KR$-norm may not be bounded from below by the Radon
  norm in general: For $\mu = \delta_{x_0} + \delta_{x_1}$ it holds that
  $\norm[\radon]{\mu} = 2$ while $\norm[\KR,\lambda]{\mu}\to 0$
  for $\abs{x_0 -x_1}\to 0$.
\end{remark}

\section{Primal formulations of the $\KR$-norm}
\label{sec:primalKR}

We present two reformulations of the $\KR$-norm. The first, only shown
formally, is similar to the Kantorovich-Rubinstein duality and shows
the relation to optimal transport.

The idea is to replace the constraint $\Lip(f)\leq \lambda_2$ by a
pointwise constraint of the form
$\abs{f(x)-f(y)}\leq\lambda_2\abs{x-y}$, i.e., we have
\[
\norm[\KR,\lambda]{\mu} = \sup\set{\int f\d\mu}{\abs{f(x)}\leq\lambda_1,\
  \abs{f(x)-f(y)}\leq\lambda_2\abs{x-y}}.
\]
We express the pointwise constraints by $f(x)-\lambda_1\leq 0$,
$-f(x)-\lambda_1\leq 0$, $f(x) - f(y) - \lambda_2\abs{x-y}\leq 0$ and
$f(y)-f(x)-\lambda_2\abs{x-y}\leq 0$, introduce Lagrange
multipliers and clean up the resulting expression and finally arrive at
\begin{equation}
  \norm[\KR,\lambda]{\mu} = \inf_{\gamma \geq 0} \Bigg[\lambda_1\int_\Omega\d\abs{\mu -
    \proj_1\gamma + \proj_2\gamma} +
  \lambda_2\int_{\Omega\times\Omega}\abs{x-y}\d \gamma\Bigg].\label{eq:kr-transportx}
\end{equation}
% This should be compared to the Wasserstein distance
% from~\eqref{eq:p-wasserstein}: \added[id=jan]{(This might be another reason why we might need a  $\gamma \geq 0$ constraint in \eqref{eq:p-wasserstein})} There $\mu$ and $\nu$ are non-negative
% measures with the same mean, and one has an equality constraint on
% the marginals of $\gamma$.

% \added[id=jan]{I think instead of \eqref{eq:p-wasserstein} this is closer to the Kantorovich-Rubinstein distance (Kantorovich-Rubinstein transshipment problem in Rachev/Rüschendorf (1.1.3) (4.1.3)),}
% \begin{equation}
% d(\nu,\nu') = \inf_{\gamma \geq 0} \int_{\Omega \times \Omega} |x-y| d \gamma\quad{s.t.}\quad\proj_1 \gamma - \proj_2 \gamma = \nu - \nu'.
% \end{equation}
% \added[id=jan]{If we set $\nu=\mu^+$ and $\nu'=\mu^-$ for a signed measure $\mu$ we get}
This expression may be compared to the following variant from \cite{rachev1998masstransportation}
\begin{equation*}
  \norm[\KR']{\mu} = \inf_{\gamma \geq 0} \set{ \int_{\Omega \times \Omega} |x-y| \d \gamma}{\proj_1 \gamma - \proj_2 \gamma = \mu}, 
\end{equation*}
which is a ``strict constraint'' version of
\eqref{eq:kr-transportx}. Because we have a metric cost function
$(x,y)\mapsto |x-y|$, this is the same as requiring $\proj_1 \gamma =
\mu^+, \proj_2 \gamma = \mu^-$ and we recover the Wasserstein metric
with $p=1$ from~\eqref{eq:p-wasserstein}.

We get another reformulation by dualizing the problem slightly
differently. The idea is to reformulate the constraint $\Lip(f)\leq
\lambda_2$ with the help of the distributional derivative of $f$ as
$\norm[\infty]{\abs{\grad f}}\leq\lambda_2$. This is allowed since for
bounded, convex and open domains $\Omega$, it is indeed the case that
$\norm[\infty]{\abs{\grad f}} = \Lip(f)$
(cf.~\cite[Prop. 2.13]{ambrosio2000functions}).
Through this reformulation, the KR-norm can be seen to be equivalent
to the flat norm in the theory of currents \cite{morgan2000gmt,federer1969gmt}.
\begin{lemma}
  \label{lem:cascading}
  Let $\Omega\subset\RR^n$ be open, convex, and bounded, and let $\lambda =
  (\lambda_1,\lambda_2)\geq 0$. Then it holds that
  \begin{equation}
  \label{eq:equiv-closure}
  \norm[\KR,\lambda]{\mu} =
  \min_{\vec{\nu}\in\radon(\closure{\Omega},\RR^n)}
  \lambda_1\norm[\radon]{\mu - \Div\vec{\nu}} +
  \lambda_2\norm[\radon]{\abs{\vec{\nu}}}
  \end{equation}
  where $\Div \vec{\nu}$ is understood to be taken in $\closure\Omega$ or,
  equivalently, in any open set $U$ containing $\closure\Omega$.
\end{lemma}
\begin{proof}
  We have
  \[
  \norm[\KR,\lambda]{\mu} = \sup_f\int_\Omega f\d\mu -
  \ind{\sett{\norm[\infty]{\cdot}\leq\lambda_1}}(f) -
  \ind{\sett{\norm[\infty]{\abs{\cdot}}\leq\lambda_2}}(\grad f).
  \]
  Now let $U$ be an open set containing $\closure\Omega$, define the
  Banach spaces $X = C_c^1(U)$ and $Y = C_0(U,\RR^n)$, and the subsets
  \begin{align*}
    A & = \set{f\in X}{\sup_{x\in\closure{\Omega}}\abs{f(x)}\leq\lambda_1}\\
    B & = \set{\vec{g}\in Y}{\sup_{x\in\closure{\Omega}}\abs{\vec{g}(x)}\leq\lambda_2}.
  \end{align*}
  Further define functionals $F:X\to\RR\union\sett{\infty}$ and
  $G:Y\to\RR\union\sett{\infty}$ by
  \[
  F(f) = -\int_\Omega f\d\nu + \ind{A}(f),\qquad G(\vec{g}) = \ind{B}(\vec{g})
  \]
  as well as the linear operator $K = \grad:X\to Y$. With this
  notation we have
  \[
  \norm[\KR,\lambda]{\mu} = \sup_{f\in X} F(f) + G(Kf).
  \]
  To use the Fenchel-Rockafellar duality~\cite{ekeland1999convex} we use the constraint
  qualification from~\cite{attouch1986dualitysumconvex}, i.e., that it holds that
  \[
  \Union_{\alpha>0}\alpha[\dom(G) - K\dom(F)] \supset
  \Union_{\alpha>0} \alpha A = Y.
  \]
  Hence, we have
  \[
  \sup_{f\in X} -F(f) - G(Kf) = \inf_{\nu\in Y^*} F^*(-K^*\nu) + G^*(\nu).
  \]
  We have $X^* = \radon(U)$ and $Y^* = \radon(U,\RR^n)$ and the
  conjugate functions of $F$ and $G$ are expressed with the help of the
  sets
  \begin{align*}
    C & = \set{\eta\in\radon(U)}{\abs{\eta}(U\setminus\closure\Omega) = 0}\\
    D & = \set{\vec{\nu}\in\radon(U,\RR^n)}{\abs{\vec{\nu}}(U\setminus\closure\Omega) = 0}
  \end{align*}
  as
  \[
  F^*(\eta) = \lambda_1\norm[\radon(\closure\Omega)]{\mu+\eta} +
  \ind{C}(\eta),\qquad G^*(\vec{\nu}) =
  \lambda_2\norm[\radon(\closure\Omega)]{\abs{\vec{\nu}}} + \ind{D}(\vec{\nu})
  \]
  Since by the Kirszbraun theorem every $f$ that is Lipschitz
  continuous on $\Omega$ can be extended to $U$ (with preservation of
  the Lipschitz constant) it follows with $K^* = -\Div:Y^*\to X^*$ that
  \begin{align*}
    \norm[\KR,\lambda]{\mu} & = \inf_{\vec{\nu}\in Y^*}F^*(-K^*\vec{\nu}) + G^*(\vec{\nu})\\
    & = \inf_{\nu\in\radon(U,\RR^n)} \lambda_1 \norm[\radon(U)]{\mu-\Div\vec{\nu}}
    + \lambda_2\norm[\radon(U)]{\abs{\vec{\nu}}} +
    \ind{C}(\Div\vec{\nu})+ \ind{D}(\vec{\nu}).
  \end{align*}
  Since bounded sets in $\radon(U,\RR^n)$ are relatively weakly*
  compact, we can replace the infimum by a minimum and since $\support
  \vec{\nu}\subset\closure{\Omega}$ implies that
  $\support\Div\vec{\nu}\subset\closure{\Omega}$ we can replace
  $\radon(U,\RR^n)$ by $\radon(\closure\Omega,\RR^n)$ and drop the
  constraints $C$ and $D$ and arrive at
  \[
  \norm[\KR,\lambda]{\mu} = \min_{\nu\in\radon(\closure{\Omega},\RR^n)}
  \lambda_1 \norm[\radon(\closure\Omega)]{\mu-\Div\vec{\nu}} +
  \lambda_2\norm[\radon(\closure\Omega)]{\abs{\vec{\nu}}}
  \]
  as desired.
\end{proof}

In Theorem \ref{thm:equiv-convex-linfty} below we will prove that 
actually we can take $\vec\nu$ as an $L^1$ vector field with $L^1$ 
divergence in \eqref{eq:equiv-closure}. Namely 
$\vec\nu \in W^{1,1}(\Omega; \Div)$, where for $\Omega \subset \R^n$ an
open domain, we define
\[
    W^{1,1}(\Omega; \Div) :=
    \{\vec\nu \in L^1(\Omega; \R^n) \mid \Div \vec\nu \in L^1(\Omega) \}.
\]
As such, our result is closely
related to the work in \cite{pratelli2002regularity}, where this
$L^1$ property is proved
for
the transport density $\abs{\vec\nu}$. Our proof is however different and 
shorter, based on the following simpler geometric estimate.

\begin{lemma}
    \label{lemma:approx-parallel}
    Let $\Omega \subset \R^n$ be convex, open and bounded, and $\mu=\sum_{i=1}^N \alpha_i \delta_{x_i}$.
    Then any optimal solution $\nu$ to \eqref{eq:equiv-closure}
    has the form $\nu=\sum_{j=1}^M \beta_j \llbracket a_j, b_j \rrbracket$,
    where $a_j, b_j = x_i$ for some $i$.
    Moreover, %if $n=2$, 
    the transport rays
    $[a_j, b_j]$ are \emph{approximately parallel} in the following
    sense: there exist constants $c=c(n)$ and $\kappa=\kappa(n)$
    such that if $[a_j, b_j] \isect B(x, \rho) \ne \emptyset$ and $[a_k, b_k] \isect B(x, \rho) \ne \emptyset$ 
    with $a_j,b_j,a_k,b_k \not \in B(x, c\rho)$, then $[a_j, b_j]$ and $[a_k, b_k]$
    satisfy $a_j, b_j, a_k, b_k \in B(x,2\kappa\rho)+\R z$
    for some unit vector $z$.
\end{lemma}
\begin{proof}
    The claim that $\nu$ has the form $\nu=\sum_{j=1}^M \beta_j \llbracket a_j, b_j \rrbracket$
    is trivial, as the problem in \eqref{eq:equiv-closure} with discrete $\mu$ is a 
    simple combinatorial problem.
    
    Suppose $[a_j, b_j] \isect B(x, \rho) \ne \emptyset$ and $[a_k, b_k] \isect B(x, \rho) \ne \emptyset$,
    and that
    $a_j,b_j,a_k,b_k \not \in B(x, c\rho)$, 
    for $c$ yet to be determined.
    If $n=2$, let $\bar a_j \defeq a_j$, $\bar b_j \defeq b_j$, $\bar a_k \defeq a_k$, 
    and $\bar b_k \defeq b_k$. Also set $d \defeq 0$, and $v \defeq 0$.
    Otherwise, if $n > 2$, let $v \in \R^n$ be the vector giving the minimum distance 
    between the lines
    \[
        L_j \defeq a_j+\R(b_j-a_j),
        \quad
        \text{and}
        \quad
        L_k \defeq a_k+\R(b_k-a_k).
    \]
    %Since  $b_j+a_j$ and $b_k+a_k$ span a two-dimensional
    %subspace, 
    We may then find a plane $P \subset \R^n$
    orthogonal to $v$ such that $L_j \subset P$ and $L_k \subset v + P$.
    After rotation and translation, if necessary, we may without loss
    of generality assume that $v=(0, d) \in \R^n$ for some
    $d \in \R^{n-2}$, and
    \[
        a_j=(\bar a_j, 0),\ b_j=(\bar b_j, 0),
        \quad
        \text{and}
        \quad
        a_k=(\bar a_k, d),\ b_k=(\bar b_k, d).
    \]
    We also denote $x=(\bar x, x_0)$.
    Since  $L_j$ and $L_k$ lie on the planes $P$ and $v+P$ at a constant
    distance $\norm{d} \le 2\rho$ apart, we find that
    $\bar a_j, \bar b_j, \bar a_k, \bar b_k \not \in B(\bar x, \gamma_n c\rho)$,
    for some dimensional constant $\gamma_n \in (0, 1)$. 
    %By multiplying 
    %original $c$ by $1/\gamma$, we may without loss of generality 
    %assume $\gamma=1$, and take 
    %$\bar a_j, \bar b_j, \bar a_k, \bar b_k \not \in B(\bar x, c\rho)$.
    In fact, we may assume by shifting all of the points closer towards $x$
    that
    \[
        \bar a_j, \bar b_j, \bar a_k, \bar b_k \in \bdry B(\bar x, \gamma_n c\rho),
    \]
    This is possible with $c>1$ as the segments $[\bar a_j, \bar b_j]$ 
    and $[\bar a_k, \bar b_k]$ pass through $B(\bar x, \rho)$, and so
    we may split each segment into three parts -- two outside
    $B(\bar x, \gamma_n c\rho)$, and one inside.
    
    Let $\kappa>2$.
    Observe now that in case $n=2$ and generally for $n>2$, when looking from the
    direction $v$, we have one of the two-dimensional situation depicted in 
    Figure \ref{fig:approx-parallel}\subref{fig:approx-parallel-one} or
    \subref{fig:approx-parallel-two}. The segments
    $[\bar a_j, \bar b_j]$ and $[\bar a_k, \bar b_k]$, starting and ending
    on $\bdry B(\bar x, \gamma_n c\rho)$, both pass through approximately ($c \gg 1$) in 
           the middle of this sphere, through $\bdry B(x, \rho)$.
    They %either cross within $\bdry B(\bar x, c\rho)$, 
    are either within a cylinder of width $2\kappa\rho$,
    as in Figure \ref{fig:approx-parallel}\subref{fig:approx-parallel-two},
    or are not, as in Figure \ref{fig:approx-parallel}\subref{fig:approx-parallel-one}.

    If $\norm{a_j-a_k}<\kappa\rho$
    and $c$ is large enough that $B(x, \rho)$ reduces to almost to
    a point in comparison to $B(x, \gamma_n c\rho)$,
    then $\norm{\bar b_j-\bar b_k}<2\kappa\rho$. This is because
    both segments $[\bar a_j, \bar b_j]$ and $[\bar a_k, \bar b_k]$ also
    pass through the ball $B(x, \rho)$ 
    and so cannot diverge much on the opposite side of the ball.
    Trivially a unit vector $z$ exists, 
    such that both segments lie in the cylinder $B(x, 2\kappa\rho)+\R z$.
    Otherwise, for large enough $c$, both $\abs{\bar a_j-\bar a_k} \ge \kappa\rho$ 
    as well as $\abs{\bar b_j-\bar b_k}>\kappa\rho$.
    Since $d \le 2\rho < \kappa\rho$, i.e., some midpoints of the segments
    are closer than the end points, we observe that the two segments
    have to cross.
    That is $[\bar a_j, \bar b_j] \isect [\bar a_k, \bar b_k]=\bar q$
    for some $\bar q$. If $c$ and $\kappa$ are large enough that
    $B(x, \rho)$ reduces to a point in comparison to everything else,
    we can make $\bar q \in B(\bar x, \rho)$. 
    By simple geometrical reasoning, on the triangle
    $\bar a_j - \bar q - \bar b_k$,
    compare Figure \ref{fig:approx-parallel}\subref{fig:approx-parallel-improve},
    it now follows that
    \[
        \abs{\bar a_j-\bar b_k}
        \le
        \sqrt{\abs{\bar a_j-\bar q}^2-(\kappa-2)^2\rho^2}
        +
        \sqrt{\abs{\bar b_k-\bar q}^2-(\kappa-2)^2\rho^2}.
    \]
    Likewise
    \[
        \abs{\bar a_k-\bar b_j}
        \le
        \sqrt{\abs{\bar a_k-\bar q}^2-(\kappa-2)^2\rho^2}
        +
        \sqrt{\abs{\bar b_j-\bar q}^2-(\kappa-2)^2\rho^2}.
    \]
    If $n=2$, or more generally $d=0$, it trivially follows that
    \[
        \begin{split}
        \abs{a_j-b_k}+\abs{a_k-b_j}
        &
        <
        \abs{a_j-q}+\abs{b_k-q}
        +
        \abs{a_k-q}+\abs{b_j-q}
        \\
        &
        =
        \abs{a_j-b_j} + \abs{a_k-b_k}.
        \end{split}
    \]
    Otherwise, minding that $\abs{d} \le 2\rho$ and $\kappa > 2$, 
    we calculate
    \[
        \begin{split}
        \abs{a_j-b_k}+\abs{a_j-b_k}
        &
        =
        \sqrt{\abs{\bar a_j-\bar b_k}^2+\abs{d}^2}
        +
        \sqrt{\abs{\bar a_k-\bar b_j}^2+\abs{d}^2}
        \\
        &
        \le
        \sqrt{(\abs{\bar a_j-\bar q}+\abs{\bar b_k-\bar q})^2-2(\kappa-2)^2\rho^2+d^2}
        \\
        &\phantom{\le}
        +
        \sqrt{(\abs{\bar a_k-\bar q}+\abs{\bar b_j-\bar q})^2-2(\kappa-2)^2\rho^2+d^2}
        \\
        &
        <
        \abs{a_j-q}+\abs{b_k-q}
        +
        \abs{a_k-q}+\abs{b_j-q}
        \\
        &
        =
        \abs{a_j-b_j} + \abs{a_k-b_k}.
        \end{split}
    \]
    This provides a contradicion to the optimality of the transport rays
    $[a_j, b_j]$ and $[a_k, b_k]$, and shows the claim.
\end{proof}

\begin{remark}
    If $n=2$, we can take $\kappa=2$, and the argument is simplified
    considerably.
\end{remark}

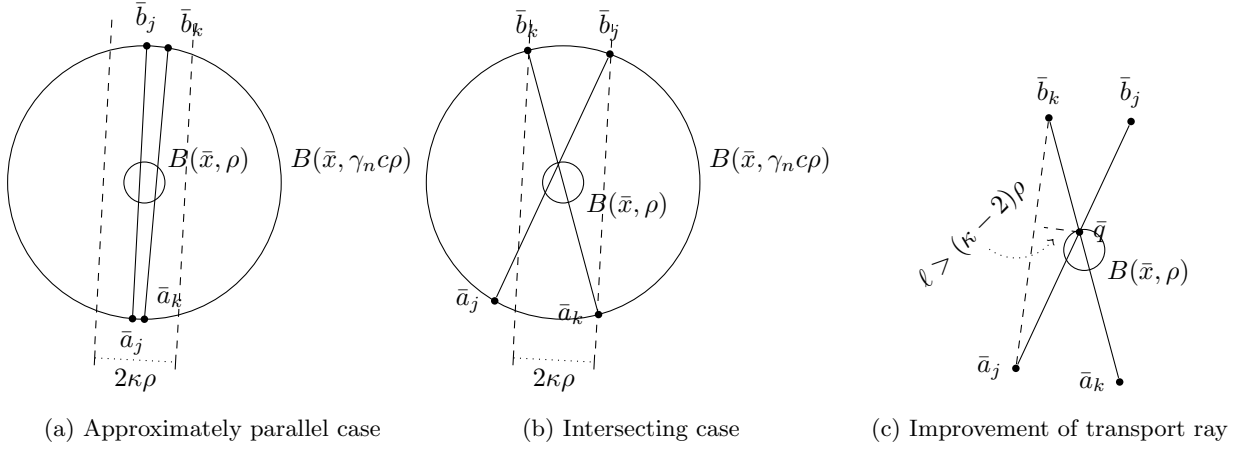
\begin{figure}
    \tikzexternaldisable
    \centerfloat%\centering
    \newlength{\w}
    \setlength{\w}{0.15\columnwidth}
    \begin{subfigure}[t]{0.45\columnwidth}
        \centering
        \begin{tikzpicture}
            \draw (0,0) circle (0.15\w) node[above right, xshift=.1\w] {$B(\bar x, \rho)$};
            \draw (0,0) circle (\w) node[above right, xshift=\w] {$B(\bar x, \gamma_n c\rho)$};
            \draw (265:\w) node[shape=circle,inner sep=1pt,fill=black,label=below:$\bar a_j$] {} -- (89:\w) node[shape=circle,inner sep=1pt,fill=black,label=above:$\bar b_j$] {};
            \draw (270:\w) node[shape=circle,inner sep=1pt,fill=black,label=above right:$\bar a_k$] {} -- (80:\w) node[shape=circle,inner sep=1pt,fill=black,label=above right:$\bar b_k$] {};
            \draw[dashed,rotate=-3,xshift=-0.3\w] (0,-1.2\w) -- (0, 1.2\w);
            \draw[dashed,rotate=-3,xshift=0.3\w] (0,-1.2\w) -- (0, 1.2\w);
            \draw[|-|,dotted, rotate=-3] (-0.3\w,-1.3\w) -- node[below] {$2\kappa\rho$}  (0.3\w, -1.3\w);
        \end{tikzpicture}
        \caption{Approximately parallel case}
        \label{fig:approx-parallel-one}
    \end{subfigure}
    \begin{subfigure}[t]{0.45\columnwidth}
        \centering
        \begin{tikzpicture}
            \draw (0,0) circle (0.15\w) node[below right, xshift=.1\w] {$B(\bar x, \rho)$};
            \draw (0,0) circle (\w) node[above right, xshift=\w] {$B(\bar x, \gamma_n c\rho)$};
            \node[shape=circle,inner sep=1pt,fill=black,label=left:$\bar a_j$] (aj) at (240:\w) {};
            \node[shape=circle,inner sep=1pt,fill=black,label=above:$\bar b_j$] (bj) at (70:\w) {};
            \node[shape=circle,inner sep=1pt,fill=black,label=left:$\bar a_k$] (ak) at (285:\w) {};
            \node[shape=circle,inner sep=1pt,fill=black,label=above:$\bar b_k$] (bk) at (105:\w) {};
            \draw (aj) -- (bj);
            \draw (ak) -- (bk); 
            \draw[dashed,rotate=-3,xshift=-0.3\w] (0,-1.2\w) -- (0, 1.2\w);
            \draw[dashed,rotate=-3,xshift=0.3\w] (0,-1.2\w) -- (0, 1.2\w);
            \draw[|-|,dotted, rotate=-3] (-0.3\w,-1.3\w) -- node[below] {$2\kappa\rho$}  (0.3\w, -1.3\w);
        \end{tikzpicture}
        \caption{Intersecting case}
        \label{fig:approx-parallel-two}
    \end{subfigure}
    \begin{subfigure}[t]{0.45\columnwidth}
        \centering
        \begin{tikzpicture}
            \draw (0,0) circle (0.15\w) node[below right, xshift=.1\w] {$B(\bar x, \rho)$};
            %\draw (0,0) circle (\w) node[above right, xshift=\w] {$B(x, c\rho)$};
            %\draw (240:\w) node[shape=circle,inner sep=1pt,fill=black,label=left:$\bar a_j$] {} -- (70:\w) node[shape=circle,inner sep=1pt,fill=black,label=above:$\bar b_j$] {};
            \node[shape=circle,inner sep=1pt,fill=black,label=left:$\bar a_j$] (aj) at (240:\w) {};
            \node[shape=circle,inner sep=1pt,fill=black,label=above:$\bar b_j$] (bj) at (70:\w) {};
            \node[shape=circle,inner sep=1pt,fill=black,label=left:$\bar a_k$] (ak) at (285:\w) {};
            \node[shape=circle,inner sep=1pt,fill=black,label=above:$\bar b_k$] (bk) at (105:\w) {};
            \draw[name path=l1] (aj) -- (bj);
            \draw[name path=l2] (ak) -- (bk); 
            \path[name intersections={of=l1 and l2,by=q}];
            \node[shape=circle,inner sep=1pt,fill=black,label=right:$\bar q$] at (q) {};
            \draw[dashed] (aj) -- node[pos=0.57](m) {} (bk);
            \draw[dashed] (q) -- node[midway](x) {} (m);
            \draw (-.7\w, 0) node[above,rotate=45] {$\ell>(\kappa-2)\rho$}
                edge[bend right=45,dotted,->] (x);
            %\draw[dashed,rotate=-3,xshift=-0.3\w] (0,-1.2\w) -- (0, 1.2\w);
            %\draw[dashed,rotate=-3,xshift=0.3\w] (0,-1.2\w) -- (0, 1.2\w);
            %\draw[|-|,dotted, rotate=-3] (-0.3\w,-1.3\w) -- node[below] {$2\kappa\rho$}  (0.3\w, -1.3\w);
        \end{tikzpicture}
        \caption{Improvement of transport ray}
        \label{fig:approx-parallel-improve}
    \end{subfigure}
    \caption{Illustration of the two-dimensional projection in the proof
        of Lemma \ref{lemma:approx-parallel}.}
    \label{fig:approx-parallel}
    \tikzexternalenable
\end{figure}

\begin{comment}    
    If $n=2$, we observe that either $a_j, b_j, a_k, b_k$ all lie on a line, or
    $(a_j, b_j) \isect (a_k, b_k) = \emptyset$, because if the segments
    $[a_j, b_j]$ and $[a_k, b_k]$ cross, then
    $\llbracket a_j, b_k \rrbracket + \llbracket a_k, b_j \rrbracket$ would 
    provide a more optimal transport compared to
    $\llbracket a_j, b_j \rrbracket + \llbracket a_k, b_k \rrbracket$.
    As a consequence, the claimed constant $c > 0$ exists.
    Indeed, suppose $[a_j, b_j] \isect B(x, \rho) \ne \emptyset$, and
    $a_j, b_j$ are outside the smallest cube $Q$ containing $B(x, \rho)$
    and aligned with $z=(b_j-a_j)/\norm{b_j-a_j}$. 
    We may find a maximum radius $c\rho \ge 2\rho$ such that if 
    $a_k \not \in B(x, \rho)+\R z$ and $a_k, b_k \not \in B(x, 2\rho)$, but
    $[a_k, b_k] \isect B(x, \rho) \ne \emptyset$, 
    then $[a_k, b_k]$ has to intersect within $B(x, c\rho)$ any (infinite or long enough) 
    line in direction $z$ passing through $B(x, \rho)$; 
    in particular $[a_j, b_j]$. But intersection is not possible, 
    so this is the $c$ we want:
\end{comment}

\begin{theorem}
    \label{thm:equiv-convex-linfty}
    Suppose $\Omega \subset \R^n$ %with $n=2$ 
    is convex, open, and bounded, and $\mu \in L^1(\Omega)$.
    Then
    \begin{equation}
      \norm[\KR,\lambda_1,\lambda_2]{\mu} = \min_{\nu \in W^{1,1}(\Omega; \Div)}
      \lambda_1\norm[L^1(\Omega; \R^n)]{\mu - \Div \nu} +
      \lambda_2\norm[L^1(\Omega)]{\nu}.
    \end{equation}
    Moreover the minimum is reached by $\nu$ satisfying $\int_\Omega \Div \nu \d \Lebesgue^n=0$.
\end{theorem}

\begin{proof}
    We assume first that $\mu \in L^\infty(\Omega)$.
    By Lemma \ref{lem:cascading}, we have \eqref{eq:equiv-closure}.
    To replace $\closure \Omega$ by $\Omega$,
    we just have to show that that $\abs{\nu}(\bdry \Omega) = 0$ for any $\nu$ reaching
    the minimum in \eqref{eq:equiv-closure}. This follows if $\nu \ll \Lebesgue^n$.
    Hence it suffices to show that actually $\nu$ and $\Div \nu$ are also absolutely continuous
    with respect to $\Lebesgue^n$.
    This is where we need the convexity of $\Omega$ and the absolute continuity of $\mu$.
    
    Clearly by \eqref{eq:equiv-closure} we have
    \begin{equation*}
      \norm[\KR,\lambda_1,\lambda_2]{\mu} 
      \le \min_{\nu \in W^{1,1}(\Omega; \Div)}
      \lambda_1\norm[L^1(\Omega; \R^n)]{\mu - \Div \nu} +
      \lambda_2\norm[L^1(\Omega)]{\nu},
    \end{equation*}
    so it remains to show the opposite inequality.
    We approximate $\mu$ in terms of strict convergence of measures
    by $\{\mu^i\}_{i=1}^\infty$, where $\mu^i = \sum_{j=1}^{N_i} \alpha_{i,j} \delta_{x_{i,j}}$. 
    We may clearly assume that $x_{i,j} \in \Omega$, because $\abs{\mu}(\bdry \Omega)=0$ 
    by absolutely continuity.
    Moreover, given a sequence $\epsilon_i \downto 0$, we may assume that
    there exist Voronoi cells $V_{i,j} \subset B(x_{i,j}, \epsilon_i)$, such that
    $\alpha_{i,j}=\int_{V_{i,j}} \mu(x) \d x$, as well as
    \begin{equation}
        \label{eq:discr-cover}
        %\support \mu \subset \Union_{j=1}^{N_i} B(x_{i,j}, \epsilon_i).
        V_{i,j} \isect V_{i,k} = \emptyset, (i \ne k), 
        \quad
        \text{and}
        \quad
        \support \mu \subset \Union_{j=1}^{N_i} V_{i,j},
        \quad
        (i=1,\ldots,N_i).
        %B(x_{i,j}, \epsilon_i).
        %\quad
        %\sum_{i=1}^{N_i} \chi_{B(x_{i,j}, \epsilon_i)} \le c_n,
    \end{equation}
    %that there exist $V_{j,i}
    %where $c_n$ is the dimensional constant from the Besicovitch covering theorem.
    Then \eqref{eq:equiv-closure} is a finite-dimensional discrete/combinatorial
    problem, and we easily discover an optimal solution $\nu^i$.
    Because tranporting mass outside $\Omega$ incurs a cost on $\bdry \Omega$,
    we see that
    \[
        \nu^i = \sum_{j=1}^{M_i} \beta_{i,j} \llbracket a_{i,j} , b_{i,j} \rrbracket,
    \]
    for some $\beta_{i,j} > 0$ and $a_{i,j}, b_{i,j} \in \{x_{i,1}, \ldots, x_{i,N_i}\}$.
    We calculate
    \[
        \Div \llbracket a , b \rrbracket = \delta_b - \delta_a.
    \]
    Moreover
    \begin{equation}
        \label{eq:nui-ac}
        \Div \nu^i(\closure \Omega)=\Div \nu^i(\Omega)=0,
        \quad
        \text{and}
        \quad
        \Div \nu^i \ll \abs{\mu^i}.
    \end{equation}
    As minimisers, we have
    \[
        \norm[\radon(\closure \Omega;\R^n)]{\nu^i} \le \frac{\lambda_1}{\lambda_2} \norm[\radon(\closure \Omega)]{\mu^i} \le \frac{\lambda_1}{\lambda_2} \norm[\radon(\closure \Omega)]{\mu}.
    \]
    Therefore, after possibly moving to a subsequence, unrelabelled, we may assume
    that $\nu^i \weaktostar \nu$ for some $\nu \in \radon(\closure \Omega; \R^n)$.
    But by \eqref{eq:nui-ac} we may also assume that $\Div \nu^i \weaktostar \lambda \in \radon(\closure \Omega)$,
    where $\lambda \ll \abs{\mu}$. From this absolute continuity it follows that
    $\lambda(\closure \Omega)=0$. (A priori it might be that $\lambda(\closure \Omega) \ne 0$.)
    Necessarily $\lambda=\Div \nu$, so that in particular $\Div \nu \ll \Lebesgue^n$.
    Because $\bdry \Omega$ is $\Lebesgue^n$-negligible,
    it follows that $\Div \nu(\Omega)=0$.
    
    %Indeed, with $i$ fixed, let
    %\[
    %    V_1 \defeq B(x_{i,1}, \epsilon_i) \isect \Omega,
    %    \quad
    %    \text{and}
    %    \quad
    %    V_j \defeq B(x_{i,j}, \epsilon_i) \isect \Omega \setminus \Union_{k=1}^{j-1} V_k,
    %    \quad
    %    (j=2,\ldots,N_i).
    %\]
    We want to show that $\nu$ is an optimal solution to \eqref{eq:equiv-closure} for $\mu$.
    We do this as follows. With $i$ fixed, within each $V_{i,j}$, ($j=1,\ldots,N_i$), 
    we may construct a map $\nu_{i,j}$ transporting the mass of $\mu$ within the cell
    $V_{i,j}$ to the cell centre $\delta_{x_{i,j}}$, or the other way around. That is
    \[
        \Div \nu_{i,j} = \mu \chi_{V_{i,j}} - \alpha_{i,j} \delta_{x_{i,j}}
    \]
    with
    \[
        \norm{\nu_{i,j}} \le \epsilon_i \int_{V_{i,j}} \abs{\mu(x)} \d x.
    \]
    %(Compare \eqref{eq:transport-to-point}.)
    %Since every point $x \in \Omega$ is covered by at most $c_n$ balls $B(x_{i,j}, \epsilon_i)$,
    It follows that 
    \[
        \sum_{j=1}^{N_i} \norm{\nu_{i,j}} \le \epsilon_i \norm{\mu}.
    \]
    If now $\nu^*$ is an optimal solution to \eqref{eq:equiv-closure} for $\mu$,
    defining
    \[
        \nu^i_0 \defeq \nu^* - \sum_{j=1}^{N_i} \nu_{i,j},
    \]
    we see that
    \[
        \norm[\radon(\closure \Omega)]{\nu^i_0}
        \le
        \norm[\radon(\closure \Omega)]{\nu^*} + C \epsilon_i
    \]
    and
    \[
        \Div \nu^i_0 = \Div \nu^*  - \mu + \mu^i
    \]
    Thus
    \[
      \begin{split}
      \lambda_1\norm[\radon(\closure \Omega)]{\mu^i - \Div \nu^i} +
      &
      \lambda_2\norm[\radon(\closure \Omega; \R^n)]{\nu^i}
      \\
      &
      \le
      \lambda_1\norm[\radon(\closure \Omega)]{\mu^i - \Div \nu^i_0} +
      \lambda_2\norm[\radon(\closure \Omega; \R^n)]{\nu^i_0}
      \\
      &
      \le
      \lambda_1\norm[\radon(\closure \Omega)]{\mu - \Div \nu^*} +
      \lambda_2\norm[\radon(\closure \Omega; \R^n)]{\nu^*} + C \epsilon_i.
      \end{split}
     \]
    By weak* lower semicontinuity
    \[
      \begin{split}
      \lambda_1\norm[L^1(\Omega)]{\mu - \Div \nu} +
      &
      \lambda_2\norm[\radon(\closure \Omega; \R^n)]{\nu}
      \\
      &
      \le
      \liminf_{i \to \infty}
      \bigl(
      \lambda_1\norm[\radon(\closure \Omega)]{\mu^i - \Div \nu^i} +
      \lambda_2\norm[\radon(\closure \Omega; \R^n)]{\nu^i}
      \bigr)      
      \\
      &
      \le
      \liminf_{i \to \infty}
      \bigl(
      \lambda_1\norm[\radon(\closure \Omega)]{\mu - \Div \nu^*} +
      \lambda_2\norm[\radon(\closure \Omega; \R^n)]{\nu^*} + C \epsilon_i
      \bigr)
      \\
      &
      =
      \lambda_1\norm[\radon(\closure \Omega)]{\mu - \Div \nu^*} +
      \lambda_2\norm[\radon(\closure \Omega; \R^n)]{\nu^*}.
      \end{split}
    \]
    Thus $\nu$ is an optimal solution to \eqref{eq:equiv-closure} for $\mu$.
    Exploiting lower semicontinuity of both of the terms, 
    we moreover see that
    $\lim_{i \to \infty} \norm[\radon(\closure \Omega; \R^n)]{\nu^i} = \norm[\radon(\closure \Omega; \R^n)]{\nu}$.
    Thus
    $\{\nu^i\}_{i=1}^\infty$ converge to $\nu$ strictly in $\radon(\closure \Omega; \R^n)$.
    Likewise $\{\mu^i-\Div\nu^i\}_{i=1}^\infty$ converge to $\mu-\Div \nu$ strictly 
    in $\radon(\closure \Omega)$. But $\{\mu^i\}_{i=1}^\infty$ were already
    constructed to converge strictly to $\mu$, and we have above seen that
    $(\Div \nu^i)^\pm \le (\mu^i)^\pm$. Therefore also
    $\{\Div\nu^i\}_{i=1}^\infty$ converge to $\Div \nu$ strictly 
    in $\radon(\closure \Omega)$.

    It remains to show that $\nu \in W^{1,1}(\Omega; \Div)$.
    We have already shown $\Div \nu \ll \Lebesgue^n \restrict \Omega$, so that
    $\Div \nu \in L^1(\Omega)$. We just have to show that
    $\nu \ll \Lebesgue^n \restrict \Omega$ to show that $\nu \in L^1(\Omega; \R^n)$.
    We do this by bounding the $n$-dimensional density of $\nu$ at each point.
    Let $M \defeq \norm[L^\infty(\Omega)]{\mu}$.
    We now refer to Lemma \ref{lemma:approx-parallel}, and approximate the mass of the
    set of approximately parallel transport rays passing through $B(x, \rho)$ by
    \[
        \begin{split}
        %\max_{\norm{z}=1} & \sum_{a_{i,j}-b_{i,j} \propto z} \beta_{i,j} \Hausdorff^1(B(x, \rho) \isect [a_{i,j}, b_{i,j}])
        %\\
        %&
        %\le
        \max_{\norm{z}=1} & \sum_{a_{i,j}, b_{i,j} \in (B(x, \kappa\rho)+\R z) \isect \Omega} \beta_{i,j} \Hausdorff^1(B(x, \rho) \isect [a_{i,j}, b_{i,j}])
        \\
        &
        \le
        \max_{\norm{z}=1} \sum_{a_{i,j}, b_{i,j} \in (B(x, \kappa\rho)+\R z) \isect \Omega} \beta_{i,j} 2\rho
        \\
        &
        \le
        \max_{\norm{z}=1} \sum_{x_{i,j} \in (B(x, \kappa\rho)+\R z) \isect \Omega} \abs{\alpha_{i,j}} 2\rho
        \\
        &
        \le
        2\rho \max_{\norm{z}=1} \sum_{x_{i,j} \in (B(x, \kappa\rho)+\R z) \isect \Omega} \int_{V_{i,j}} \abs{\mu(y)} \d y
        \\
        &
        \le
        2\rho \max_{\norm{z}=1} \int_{B(x, \kappa\rho+\epsilon_i)+z\R} \abs{\mu(y)} \d y
        \\
        &
        \le
        2\rho(\kappa\rho+\epsilon_i)^{n-1} \diam(\Omega) M
        \end{split}
    \]
    Also the mass of the set of transport rays with start or end point in $B(x, c\rho)$ may be approximated by
    \[
        \begin{split}
        \sum_{a_{i,j} \in B(x, c\rho))} & \beta_{i,j} \Hausdorff^1(B(x, \rho) \isect [a_{i,j}, b_{i,j}])
        +
        \sum_{b_{i,j} \in B(x, c\rho))} \beta_{i,j} \Hausdorff^1(B(x, \rho) \isect [a_{i,j}, b_{i,j}])
        \\
        &
        \le
        \sum_{x_{i,j} \in B(x, c\rho))} 4 \alpha_{i,j} \rho
        \\
        &
        =
        \sum_{x_{i,j} \in B(x, c\rho))} 4 \rho \int_{V_{i,j}} \abs{\mu(y)} \d y
        \\
        &
        \le
        4 \rho \int_{B(x, c\rho+\epsilon_i)} \abs{\mu(y)} \d y.
        \end{split}
    \]
        
    It now follows that
    \[
        \abs{\nu^i}(B(x, \rho))
        \le 
        4 \rho \int_{B(x, c\rho+\epsilon_i)} \abs{\mu(y)} \d y
        +
        2\rho(2\kappa\rho+\epsilon_i)^{n-1} \diam(\Omega) M
    \]
    Letting $i \to \infty$, we get by lower semicontinuity
    \[
        \abs{\nu}(B(x, \rho))
        \le 
        4 \rho \int_{B(x, c\rho)} \abs{\mu(y)} \d y
        +
        2^n\kappa^{n-1}\rho^n \diam(\Omega) M
    \]
    Thus
    \[
        \lim_{\rho \downto 0} 
        \frac{\abs{\nu}(B(x, \rho))}
             {\Lebesgue^n(B(x, \rho))}
        \le
        0
        +
        2^n \kappa^{n-1} \diam(\Omega) M
    \]
    It follows (see \cite[Theorem 2.12]{mattila1999geometry})
    that $\nu \ll \Lebesgue^n \restrict \Omega$
    with
    \[
        \norm[L^1(\Omega; \R^n)]{\nu} \le 2^n \kappa^{n-1}\diam(\Omega) M \Lebesgue^n(\Omega).
    \]
    
    Finally, we consider the case of unbounded $\mu \in L^1(\Omega)$.
    We take
    \[
        \mu_M(x) \defeq \max\{-M, \min\{\mu(x), M\}\},
        \quad
        (M=1,2,3,\ldots).
    \]
    Then $\mu_M^\pm \le \mu^\pm$.
    Applying the point-mass approximation above to both $\mu^k$ and $\mu$,
    we can take $(\mu_M^i)^\pm \le (\mu^i)^\pm$.
    Then by a simple argument we also have 
    $\abs{\nu_M^i} \le \abs{\nu^i}$ for each $i,k=1,2,3,\ldots$;
    compare \cite[Proposition 4.3]{pratelli2002regularity}.
    Indeed, let $\tilde \mu_M^i \defeq \Div \nu_M^i$.
    Clearly
    \[
        (\tilde \mu_M^i)^\pm \le (\mu_M^i)^\pm \le (\mu^i)^\pm.
    \]
    We can therefore find a measure $\tau_M^i \in \radon(\Omega; \R^n)$ 
    with $\abs{\tau_M^i} \le \abs{\nu^i}$ 
    such that $\Div \tau_M^i = \tilde \mu_M^i$. If $\tau_M^i$ is not optimal,
    then we find a contradiction to $\nu^i$ being optimal by replacing 
    it with
    $\nu^i + \nu_M^i - \tau_M^i$.
    We may therefore assume that $\nu_M^i = \tau_M^i$.
    Consequently $\abs{\nu_M^i} \le \abs{\nu^i}$.
    Similarly we prove that $\abs{\nu_M^i} \le \abs{\nu_{M+1}^i}$.
    By the strict convergence of $\nu^i$ to $\nu$,
    we now deduce that $\abs{\nu_M} \le \abs{\nu}$ and
    $\abs{\nu_M} \le \abs{\nu_{M+1}}$.
    By an analogous argument we prove that
    $(\Div \nu_M^i)^\pm \le (\Div \nu^i)^\pm$,
    $(\Div \nu_M^i)^\pm \le (\Div \nu_{M+1}^i)^\pm$,
    and consequently $(\Div \nu_M)^\pm \le (\Div \nu)^\pm$ and
    $(\Div \nu_M)^\pm \le (\Div \nu_{M+1})^\pm$.
    Also $\abs{\Div \nu_M}(\Omega) \to \abs{\Div \nu}(\Omega)$,
    because
    \[
        \norm[\radon(\Omega)]{\Div \nu-\Div \nu_M}
        \le
        \norm[\radon(\Omega)]{\mu-\mu_M}.
    \]
    (This can be verified by the point-mass approximation.)
    It follows that $\Div \nu_M \to \Div \nu$ strongly.
    In particular $\Div \nu_M-\mu_M \to \Div \nu-\mu$ strongly.
    By lower semicontinuity of
    $\norm[\KR,\lambda_1,\lambda_2]{\cdot}$
    we therefore deduce that
    $\liminf_{M\to\infty} \abs{\nu_M}(\Omega) \ge \abs{\nu}(\Omega)$.
    Since $\abs{\nu_M} \le \abs{\nu}$,
    it follows that $\nu_M \to \nu$ strongly in $\radon(\Omega;\R^n)$.
    But the above paragraphs say that $\nu_M \in L^\infty(\Omega)$.
    Thus necessarily $\nu_M \in L^1(\Omega)$.
    %Similarly we prove that $\abs{\nu_M^i} \le \abs{\nu_{M+1}^i}$.
    %It follows that $\nu_M^i \to \nu^i$ strongly in $\radon(\Omega;\R^n)$.
    %Also $\nu_M \weaktostar \nu$ weakly* in $\radon(\Omega;\R^n)$.
    %Thus by lower semicontinuity
    %\[
    %    \norm[\radon(\Omega;\R^n)]{\nu-\nu_M}
    %    \le
    %    \liminf_{i \to \infty}
    %    \norm[\radon(\Omega;\R^n)]{\nu^i-\nu_M^i}
    %    \le \diam(\Omega) \norm[\radon(\Omega)]{\mu-\mu_M}
    %\]
    %It follows that $\nu_M \to \nu$ strongly in $\radon(\Omega;\R^n)$.
    %But the above paragraphs say that $\nu_M \in L^\infty(\Omega)$.
    %Thus necessarily $\nu_M \in L^1(\Omega)$.
\end{proof}

\section{Kantorovich-Rubinstein-TV denoising}
\label{sec:KR-TV-denoising}

In this section we assume that $\Omega$ is a bounded, convex and open
domain in $\RR^n$ and study the minimization problem
\begin{equation}
  \label{eq:krtv-denoising}
  \min_u\norm[\KR,\lambda]{u-u^0} + \TV(u)
\end{equation}
for some $u^0\in L^1(\Omega)$ and $\lambda=(\lambda_1,\lambda_2)\geq
0$. We call this \emph{Kantorovich-Rubinstein-$\TV$ denoising}, or
short $\KR$-$\TV$ denoising.  Using the different forms of the
$\KR$-norm we have two different form of the $\KR$-$\TV$ denoising
problem. The first uses the definition~\eqref{eq:KR-def} but we
replace the constraint $\Lip(f)\leq\lambda_2$ with the help of the
distributional gradient as $\abs{\grad f}\leq\lambda_2$. Then
problem~\eqref{eq:krtv-denoising} has the form
\begin{equation}
  \min_u\max_{
    \begin{array}{c}
      \abs{f}\leq\lambda_1\\\abs{\grad f}\leq \lambda_2
    \end{array}
  } \int_\Omega f (u-u^0) + \TV(u).  \label{eq:krtv-saddle-1}
\end{equation}
We call this form, the \emph{primal formulation}.  Another formulation
is obtained by using 
Theorem~\ref{thm:equiv-convex-linfty} to obtain
\begin{equation}
  \label{eq:krtv-cascade}
  \min_{u,\vec{\nu}}
  \lambda_1\norm[L^1]{u-u^0-\Div\vec{\nu}} + \lambda_2\norm[L^1]{\abs{\vec{\nu}}} 
  + \TV(u).
\end{equation}
We call this the \emph{cascading} or \emph{dual formulation}. 

Note that the optimal transport formulation~\eqref{eq:kr-transportx}
will not be used any further in this paper. The reason is, that this
formulation does not seem to be suited for numerical purposes as it
involves a measure on the domain $\Omega\times\Omega$ which leads, if
discretized straightforwardly, to too large storage demands.

We denote
\begin{equation}
  \label{eq:def-H}
  H_{\lambda}(u,f) =
  \begin{cases}
    \int f (u-u^0) + \TV(u), & \text{if }\ \abs{f}\leq\lambda_1,\ \abs{\grad f}\leq\lambda_2 \\
    -\infty, & \text{otherwise.}
  \end{cases}
\end{equation}
Then,~\eqref{eq:krtv-saddle-1} reads as $\min_u\max_f
H_{\lambda_1,\lambda_2}(u,f)$.

\subsection{Relation to $L^1$-$\TV$ denoising}
\label{sec:rel-l1-tv}

Similar to~\eqref{eq:kr-specia-case} one has
$\norm[\KR,(\lambda_1,\infty)]{\mu} = \lambda_1\norm[\radon]{\mu}$ and
for $u\in L^1(\Omega)$ it holds that $\norm[\radon]{u} =
\norm[L^1]{u}$. Hence, $\KR$-$\TV$ is a generalization of the successful
$L^1$-$\TV$ denoising~\cite{chan2005aspects}:
\begin{equation}
  \min_u\norm[\KR,(\lambda_1,\infty)]{u-u^0} + TV(u) =
  \min_u\norm[L^1]{u-u^0} + \tfrac1{\lambda_1}\TV(u).\label{eq:L1-tv-denoising}
\end{equation}
We will study the influence of the additional parameter $\lambda_2$ in
Section~\ref{sec:one-dim-ex} and~\ref{sec:two-dim-KR-TV-denoise}
numerically.  Note, however, that it is possible that the minimizer
of~\eqref{eq:L1-tv-denoising} may also be a minimizer
of~\eqref{eq:krtv-denoising} for $\lambda_2$ large enough but finite:
To see this, we express $L^1$-$\TV$ as a saddle point problem by
dualizing the $L^1$ norm to obtain
\[
\min_u\max_{\abs{f}\leq\lambda_1} \int_\Omega f(u-u^0) + \TV(u).
\]
We denote by $(\bar u,\bar f)$ a saddle point for this functional. If
the function $\bar f$ is already Lipschitz continuous with constant $L$,
then $(\bar u,\bar f)$ is also a solution of the saddle point problem
\[
\min_u\max_{
  \begin{array}{c}
    \abs{f}\leq \lambda_1\\
    \Lip(f)\leq\lambda_2
  \end{array}
}
\int_\Omega f(u-u^0) + \TV(u)
% = \min_u\norm[\KR,(\lambda_1,\lambda_2)]{u-u^0} + \TV(u)
\]
for any $\lambda_2\geq L$ and consequently, $\bar u$ is a solution of
the $\KR$-$\TV$ problem.

\subsection{Relation to $\TGV$ denoising}
\label{sec:rel-tgv}

The cascading formulation~\eqref{eq:krtv-cascade} reveals an
interesting conceptional relation to the total generalized variation
(TGV) model~\cite{bredies2010tgv}. To define it, we introduce
$S^{n\times n}$ as the set of symmetric $n\times n$ matrices and for a
function $v$ with values in $S^{n\times n}$ we set
\[
(\Div v(x))_i = \sum_{j=1}^n \frac{\partial v_{ij}}{\partial x_j},\qquad \Div^2 v(x) = \sum_{i,j=1}^n \frac{\partial^2 v_{ij}}{\partial x_j\partial x_i}.
\]
The total
generalized variation of order two for a parameter
$\alpha=(\alpha_1,\alpha_2)$ is
\[
\begin{split}
  \TGV_\alpha^2(u) = \sup\Big\{\int_\Omega u\Div^2 v\d x\ :\ v\in
  C_c^2(\Omega,S^{n\times n}),\\
  \hspace*{4cm}\abs{v(x)}\leq\alpha_1,\ \abs{\Div
    v(x)}\leq\alpha_2\Big\}
\end{split}
\]
The $\TGV$ term has an equivalent reformulation as follows: Denote by
$\BD(\Omega)$ the space of vector fields of bounded deformation,
i.e. vectorfields $\vec w\in L^1(\Omega,\RR^n)$ such that the symmetrized
distributional gradient $\calE \vec w = \tfrac12(\grad \vec w + \grad \vec w^T)$ is a
$S^{n\times n}$-valued Radon measure. Then it holds that
\begin{equation*}
  \TGV_\alpha^2(u) = \inf_{\vec w\in\radon(\Omega,\RR^n)} \alpha_1\norm[\radon]{\abs{\calE \vec w}} + \alpha_2\norm[\radon]{\abs{\grad u - \vec w}}
\end{equation*}
(cf.~\cite{l1tgv,bredies2011inverse}).  Note that this reformulation
resembles the spirit of the reformulation of the
Kantorovich-Rubinstein norm from Lemma~\ref{lem:cascading}:
\[
\norm[\KR,\lambda]{\mu} =
\min_{\vec{\nu}\in\radon(\closure{\Omega},\RR^n)}
\lambda_1\norm[\radon]{\mu - \Div\vec{\nu}} +
\lambda_2\norm[\radon]{\abs{\vec{\nu}}}
\]
We obtain a new and higher-order (semi-)norm by ``cascading'' the
higher order term in a new minimization problem. In the $\TV$ case we
go from $\TV(u) = \norm[\radon]{\abs{\grad u}}$ to $\TGV_\alpha^2$ by
cascading with a vector field and penalizing the symmetrized gradient
of this vector field. In the $\KR$ case, however, we go from
$\norm[L^1]{u} = \norm[\radon]{u}$ to $\norm[\KR,\lambda]{\cdot}$ by
cascading with the divergence of a vector field and penalizing with
the Radon norm of that vector field. One may say, that $\TGV_\alpha^2$
is a higher order generalization of the total variation while the
$\KR$-norm is a \emph{lower order} generalization of the $L^1$ norm
(or the Radon norm).

\subsection{Relation to $G$-norm cartoon-texture decomposition}
\label{sec:re-g-norm}

In~\cite{meyer2002oscillating} Meyer introduced the $G$-norm as a
discrepancy term in denoising problems to allow for oscillating
patterns in the denoised images. The $G$-norm is defined as
\[
\norm[G]{u} = \inf\{\norm[\infty]{\abs{\vec{g}}}\ :\ \Div \vec{g} =
u,\ g\in L^\infty\}.
\]
Meyer proposed the following $G$-$\TV$ minimization problem
\[
\min_u \lambda\norm[G]{u-u_0} + \TV(u) =
\min_{u,\vec{g}}\lambda\norm[\infty]{\abs{\vec{g}}} + \TV(u) + \delta_{\{0\}}(\Div
\vec{g} - (u-u_0)).
\]
This differs from problem~\eqref{eq:krtv-cascade} in two aspects:
First, $\abs{\vec{g}}$ is penalized in the $\infty$-norm instead of
the $1$-like Radon norm and second, the equality $\Div \vec{g} =
u-u_0$ is enforced exactly, while in~\eqref{eq:krtv-cascade} a
mismatch is allowed. The Meyer model has also been treated in numerous
other papers, e.g.~\cite{kindermann2006BVduality,aujol2006structure,yin2007comparison,duval2010texture}.

\subsection{Properties of $\KR$-$\TV$ denoising}
\label{sec:properties-kr-tv}

Similar to the case of $L^1$-$\TV$ denoising (cf.~\cite[Lemma 5.5]{chan2005aspects}) there exist thresholds
for $\lambda_1$ and $\lambda_2$ such that the minimizer
of~\eqref{eq:krtv-denoising} is $u_0$ (if $u_0$ is regular enough in some
sense) if $\lambda_1$ and $\lambda_2$ are above the thresholds:
\begin{theorem}
  Let $u_0\in BV(\Omega)$ and assume that there exists a continuously
  differentiable vector field $\vec\phi$ with compact support such that
  \begin{enumerate}
  \item $\abs{\vec\phi}\leq 1$ and
  \item $\int u_0\Div\vec\phi = TV(u_0)$.
  \end{enumerate}
  Then there exists thresholds $\lambda_1^*$ and $\lambda_2^*$ such
  that for $\lambda_1>\lambda_1^*$ and $\lambda_2>\lambda_2^*$, the
  unique minimizer of~\eqref{eq:krtv-denoising} is $u_0$.
\end{theorem}
\begin{proof} For any $u\in BV$ we have
  \begin{align*}
    \norm[\KR,\lambda_1,\lambda_2]{u-u_0} + \TV(u) & \geq \int u\Div \vec\phi + \Big[\min_{\vec\nu} \lambda_1\norm[\radon]{u-u_0 - \Div\vec\nu} + \lambda_2\norm[\radon]{\abs{\vec\nu}}\Big]\\
    & = \int u_0\Div \phi + \min_ {\vec\nu}\Big[ \lambda_1\norm[\radon]{u-u_0 - \Div\vec\nu} + \lambda_2\norm[\radon]{\abs{\vec\nu}}\\
    & \qquad + \int (u-u_0 - \Div \vec\nu)\Div\vec\phi + \int\Div\vec\nu\Div\vec\phi\Big]\\
    & \geq \TV(u_0)  + \min_{\vec\nu} \Big[(\lambda_1 - \norm[\infty]{\Div\vec\phi})\norm[\radon]{u-u_0 - \Div\vec\nu} +\\
    & \qquad (\lambda_2 - \norm[\infty]{\abs{\grad\Div\vec\phi}})\norm[\radon]{\abs{\vec\nu}}\Big]
  \end{align*}
  Hence, the values $\lambda_1^* = \norm[\infty]{\Div\vec\phi}$ and
  $\lambda_2^* = \norm[\infty]{\abs{\grad\Div\vec\phi}}$ are valid thresholds as
  claimed.
\end{proof}

Likewise there are thresholds in the opposite direction,
again similarly to the $L^1$-$\TV$ case.

\begin{theorem}
    Let $\Omega \subset \R^n$ be a convex open domain with Lipschitz
    boundary. Then there exists a constant $C=C(\Omega)$ such that
    any solution $\bar u$ to \eqref{eq:krtv-denoising}
    is a constant whenever $1/C > \lambda_1$.
\end{theorem}

\begin{proof}
    Let $f$ maximize $H_\lambda(\bar u, \cdot)$. Define
    \[
        \tilde u(y) \defeq \dashint_\Omega \bar u(x) \d x.
    \]
    Let $\tilde f$ maximize $H_\lambda(\tilde u, \cdot)$. 
    Since $\bar u$ solves \eqref{eq:krtv-denoising}, 
    we have
    \[
        H_\lambda(\tilde u, \tilde f)
        \ge
        H_\lambda(\bar u, f).
    \]
    In other words, using $\TV(\tilde u)=0$, writing out $H_\lambda$,
    and rearranging terms
    \[
        \int_\Omega \tilde f(\bar u-u^0) \d x + \int_\Omega \tilde f(\tilde u-\bar u) \d x 
        \ge
        \int_\Omega f(\bar u-u^0) \d x + \TV(\bar u).
    \]
    But, by the choice of $f$, we have
    \[
        \int_\Omega \tilde f(\bar u-u^0) \d x 
        \le
        \int_\Omega f(\bar u-u^0) \d x.
    \]
    Therefore
    \[
        \TV(\bar u)
        \le
        \int_\Omega \tilde f(\tilde u-\bar u) \d x.
    \]
    An application of Poincar{\'e}'s inequality yields
    \[
        \TV(\bar u)
        \le
        \lambda_1 C \TV(\bar u).
    \]
    This is a contradiction unless $1 < \lambda_1 C$
    or $\TV(\bar u)=0$, i.e., $\bar u$ is a constant.
\end{proof}

The second of the above two theorems shows that for small 
$\lambda_1$ one recovers a constant solution. In fact, this has to be
$\dashint_\Omega u^0 \d x$.
The first of the above two theorems shows that for parameters $\lambda_1$ and
$\lambda_2$ large enough, one recovers the input $u^0$ from the
$\KR$-$\TV$ denoising problem. This behavior is similar to the
$L^1$-$\TV$ denoising problem. If one leaves the regime of exact
reconstruction one usually observes that for $L^1$-$\TV$ denoising
mass disappears and also the phenomenon of ``suddenly vanishing
sets'' (cf.~\cite{duval2009geometricl1tv}). In contrast, for the
$\KR$-$\TV$ denoising model, we have mass conservation of the
minimizer even in the range of parameters, where exact reconstruction
does not happen anymore and noise is being removed. The precise
statement is given in the next theorem:
\begin{theorem}[Mass preservation]\label{thm:mass-preservation}
  If $\frac{\lambda_2}{\lambda_1}\leq\frac{2}{\diam(\Omega)}$, then
  \[
  \min_u \norm[\KR,\lambda_1,\lambda_2]{u-u^0} + \TV(u)
  \]
  has a minimizer $\bar u$ such that $\int_\Omega \bar u(x)\d x =  \int_\Omega u^0(x)\d x$.
\end{theorem}
\begin{proof}
  The idea is, to prove that a minimizer of the $\KR$-$\TV$ denoising
  problem with $\lambda_1=\infty$ is also a minimizer of the problem with
  finite but large enough $\lambda_1$. Hence we start by denoting with $(\bar u,\bar
  f)$ a solution of the following saddle-point problem:
  \begin{equation}
    \label{eq:saddle_lambda1_inf}
    \min_u\max_{\abs{\grad f}\leq\lambda_2} \int f (u-u^0)\d x + \TV(u)
  \end{equation}
      With the notation \eqref{eq:def-H},
  \eqref{eq:saddle_lambda1_inf} reads as $\min_u\max_f
  H_{\infty,\lambda_2}(u,f)$.
  
  It holds that $\int_\Omega\bar u\d x = \int_\Omega u^0\d x$, because otherwise, the
  $\max$ would be $\infty$. In other words: with $\lambda_1=\infty$ we
  have mass preservation.
  
  Now let $\frac{\lambda_2}{\lambda_1}\leq \frac2{\diam(\Omega)}$. We
  aim to show that there is constant $c$ such that $(\bar u,\bar f+c)$
  is a solution of
  \begin{equation}
    \label{eq:saddle}
    \min_u\max_{
      \begin{array}{c}
        \abs{f}\leq\lambda_1,\\
        \abs{\grad f}\leq\lambda_2
      \end{array}
    } \int f(u-u^0)\d x + \TV(u).
  \end{equation}
  Since $\bar f$ is Lipschitz with constant $\lambda_2$, we get that
  $\bar f(x) - \bar f(x)\leq \lambda_2\abs{x-y}$, and hence, $\max \bar
  f - \min \bar f\leq \lambda_2\diam(\Omega)$. Consequently, there is a
  constant $c$ such that
  \[
  \abs{\bar f + c}\leq \lambda_2\frac{\diam(\Omega)}{2} \leq\lambda_1
  \]
  in other words: $\bar f + c$ is feasible for~\eqref{eq:saddle}. Since
  $\int\bar u = \int u^0$ we also have
  \[
  H_{\lambda_1,\lambda_2}(\bar u,\bar f+c) = \int \bar f(\bar u - u^0)\d x +
  c\underbrace{\int(\bar u - u^0)\d x}_{=0} + TV(u) =
  H_{\infty,\lambda_2}(\bar u,\bar f).
  \]
  Since all $f$ that are feasible for~\eqref{eq:saddle} are also
  feasible for~\eqref{eq:saddle_lambda1_inf}, we have for all these $f$
  that
  \begin{equation}
    \label{eq:saddle_upper}
    H_{\lambda_1,\lambda_2}(\bar u,f) 
    \le
    H_{\infty,\lambda_2}(\bar u,f) 
    \le
    H_{\infty,\lambda_2}(\bar u,\bar f) = H_{\lambda_1,\lambda_2}(\bar u,\bar f+c).
  \end{equation}
  Also we have by $(\bar u, \bar f)$ being a saddle-point
  for all $u$ that
  \[
  H_{\lambda_1,\lambda_2}(\bar u,\bar f + c) =
  H_{\infty,\lambda_2}(\bar u,\bar f) \leq H_{\infty,\lambda_2}(u,\bar
  f).
  \]
  But since $\TV(u) = \TV(u+d)$ for every constant $d$ we also have
  with $d = c \int(u-u^0)\d x/\int \bar f\d x$ that
  \begin{equation}
    \label{eq:saddle_lower}
    \begin{split}
      H_{\lambda_1,\lambda_2}(\bar u,\bar f + c) \leq
      H_{\infty,\lambda_2}(u + d,\bar f) = \int \bar f(u-u^0)\d x +
      d\int\bar f\d x + \TV(u)\\ = \int (\bar f+c)(u-u^0)\d x + \TV(u) =
      H_{\lambda_1,\lambda_2}(u,\bar f + c).
    \end{split}
  \end{equation}
  Together, \eqref{eq:saddle_upper} and~\eqref{eq:saddle_lower} show
  that for all $f$ and $u$ it holds that
  \[
  H_{\lambda_1,\lambda_2}(\bar u,f) \leq H_{\lambda_1,\lambda2}(\bar
  u,\bar f +c) \leq H_{\lambda_1,\lambda_2}(u,\bar f + c)
  \]
  and this shows that $(\bar u,\bar f + c)$ is a solution
  of~\eqref{eq:saddle}.
\end{proof}

Note that the above theorem remains valid if we replace the $\TV$
penalty by any other penalty that is invariant under addition of
constants such as Sobolev semi-norms.

We state a lemma on the subdifferential of the total variation of the
positive and negative part of a function which we use in the following
theorem.
\begin{lemma}
    \label{lemma:tv-subdiff-inclusion}
    Let $u \in \BVspace(\Omega)$.
    Then $\partial \TV(u) \subset \partial \TV(u^+)$
    and $\partial \TV(u) \subset \partial \TV(u^-)$.
\end{lemma}
\begin{proof}
    It suffices to prove the inclusion $\partial \TV(u) \subset \partial \TV(u^+)$,
    the other inclusion being completely analogous.
    We begin by observing that if $L \in \partial \TV(u)$, 
    as a linear functional $L^* \in [\BVspace(\Omega)]^*$, then
    \[
        \TV(u)=L(u).
    \]
    This follows from applying the definition of the subdifferential
    \begin{equation}
        \label{eq:tv-subdiff-tv}
        \TV(v) - \TV(u) \ge L(v-u), 
        \quad \text{for all } v \in \BVspace(\Omega),
    \end{equation}
    to both $v=0$ and $v=2u$. If we now apply the definition to $v=u^-$,
    and also use the fact that $-L \in \partial \TV(-u)$, we deduce
    \begin{equation}
        \label{eq:subdiff-uminus-lowerbound}
        \TV(u^-) \ge \abs{L(u^-)}.
    \end{equation}
    Using $\TV(u)=\TV(u^+)+\TV(u^-)$ to rearrange \eqref{eq:tv-subdiff-tv}, 
    we have
    \begin{equation}
        \notag
        \TV(v) - \TV(u^+) \ge L(v-u^+) + \bigl(\TV(u^-)+L(u^-)\bigr), 
        \quad \text{for all } v \in \BVspace(\Omega).
    \end{equation}
    Referring to \eqref{eq:subdiff-uminus-lowerbound} we deduce
    $L \in \partial \TV(u^+)$.
\end{proof}

\begin{theorem}[Weak maximum principle]
  \label{thm:maxpincip}
  Let $u^0\geq 0$. Then there exists a minimizer $\bar u$
  of~\eqref{eq:krtv-denoising} that also fulfills $\bar u\geq 0$.
\end{theorem}

\begin{proof}
    Writing the necessary and sufficient optimality 
    conditions for the saddle point formulation \eqref{eq:krtv-saddle-1}
    of \eqref{eq:krtv-denoising}, we have
    \cite[Theorem 4.1 \& Proposition 3.2, Chapter III]{ekeland1999convex}
    \begin{align}
        \label{eq:krtv-saddle-oc-1}
        0 & \in f + \partial \TV(\bar u), \quad \text{and} \\
        \label{eq:krtv-saddle-oc-2}
        \bar u - u_0 & \in N_{C_1}(f) + N_{C_2}(f),
    \end{align}
    where the constraint sets are
    \begin{align}
        C_1 & \defeq \{ f \in \Lip(\Omega) \mid -\lambda_1 \le f(x) \le \lambda_1 \text{ for all } x \in \Omega \},
        \quad \text{and}
        \\
        C_2 & \defeq \{ f \in \Lip(\Omega) \mid \norm{\grad f(x)} \le \lambda_2 \text{ for all } x \in \Omega \}.
    \end{align}
    Application of Lemma \ref{lemma:tv-subdiff-inclusion} shows that
    \begin{equation}
        \label{eq:krtv-saddle-oc-1plus}
        0 \in f + \partial \TV(\bar u^+),
    \end{equation}
    so
    that the first condition \eqref{eq:krtv-saddle-oc-1} is
    satisfied by $\bar u^+$ as well. Let us show that also \eqref{eq:krtv-saddle-oc-2}
    is satisfied by $\bar u^+$. To begin with we observe that at
    $\Lebesgue^n$-a.e.~point $x$ with $\bar u(x)<0$, either $C_1$ or $C_2$ 
    is active. Indeed, since $\bar u(x)-u_0(x)<0$ at such point, in the problem
    \[
        \max_{f \in C_1 \isect C_2} \int_{\Omega} f(\bar u-u_0) \d x,
    \]
    the solution $f$ should be as negative as possible within 
    the constraints. If it is as negative as possible, $C_1$ is active,
    and
    \[  
        [N_{C_1}(f)](x)=(-\infty, 0].
    \]
    Otherwise, $C_2$ has to be active, with $f$ going as fast as possible
    to the least possible value it can achieve. In this case,
    \[  
        [N_{C_2}(f)](x)=[0, \infty) \sign[-\Div \grad f(x)].
    \]
    If $C_1$ is not active, this has to be
    \[
        [N_{C_2}(f)](x)=(-\infty, 0],
    \]
    for $\bar u$ to satisfy \eqref{eq:krtv-saddle-oc-2}.
    In either case, the right hand side of \eqref{eq:krtv-saddle-oc-2}
    is $(-\infty, 0]$. Therefore, trivially
    \begin{equation}
        \label{eq:krtv-saddle-oc-2plus}
        \bar u^+ - u_0 \in N_{C_1}(f) + N_{C_2}(f)
        =(-\infty, 0].
    \end{equation}
    Combined, \eqref{eq:krtv-saddle-oc-1plus} and \eqref{eq:krtv-saddle-oc-2plus}
    show that $\bar u^+$ is a solution to \eqref{eq:krtv-saddle-1}.
\end{proof}

\begin{corollary}[Weak boundedness]
  Let $u^0 \in L^\infty(\Omega)$. 
  Then there exists a solution $\bar u$ 
  of~\eqref{eq:krtv-denoising} fulfilling
  $\norm[L^\infty(\Omega)]{\bar u} \le \norm[L^\infty(\Omega)]{u^0}$.
\end{corollary}

\begin{proof}
    The problem \eqref{eq:krtv-denoising} is affine-invariant,
    i.e., for data $a u^0+c$ for any constants $a, c \in \R$
    we have $a \bar u + c$ as a solution. Setting $M \defeq \norm[L^\infty(\Omega)]{u^0}$
    and applying Theorem \ref{thm:maxpincip} to data
    $u^0 + M$ and $-u^0 + M$ proves the claim.
\end{proof}

\begin{corollary}[No negative solution if mass is preserved]
  If $u^0\geq 0$ and $\tfrac{\lambda_2}{\lambda_1}\leq
  \tfrac{2}{\diam(\Omega)}$ then any minimizer
  of~\eqref{eq:krtv-denoising} is non-negative.
\end{corollary}
\begin{proof}
  The proof of Theorem~\ref{thm:maxpincip} reveals that if $\bar u$ is
  a solution, then also $\bar u^+$ is a solution. However, if $\bar u$
  would have a negative part (i.e. $\int_\Omega \bar u^-\d x>0$) then
  $\bar u$ and $\bar u^+$ would have a different mean value which
  would contradict Theorem~\ref{thm:mass-preservation}.
\end{proof}

\section{Numerical solution}
\label{sec:numerical-solution}

In this section we briefly sketch how one may solve the $\KR$-$\TV$
denoising problem~\eqref{eq:krtv-denoising} numerically. Basically, we
rely on methods to solve convex-concave saddle point problems, see,
e.g.~\cite{chambolle2011first,lorenz2014accelerated,goldstein2013adaptive}.

For the primal formulation~\eqref{eq:krtv-denoising} with Lipschitz constraint we
reformulate as follows:
\begin{align}
  &\min_u\max_{f,\phi} \int_\Omega f (u-u^0)\d x + \int_\Omega\grad u\cdot\phi \d x - I_{\norm[\infty]{\cdot}\leq 1}(\abs{\phi})\nonumber\\
  &\qquad\qquad- I_{\norm[\infty]{\cdot}\leq\lambda_1}(f) - I_{\norm[\infty]{\cdot}\leq \lambda_2}(\grad f)
\end{align}
By dualizing the term $I_{\norm[\infty]{\cdot}\leq \lambda_2}(\grad f)$ we obtain another primal variable $q$ and end up with
\begin{equation}
  \begin{split}
    &\min_{u,q}\max_{f,\phi} \int_\Omega f (u-u^0)\d x +
    \int_\Omega\grad u\cdot\phi \d x \\
    &\qquad\qquad- I_{\norm[\infty]{\cdot}\leq 1}(\abs{\phi}) -
    I_{\norm[\infty]{\cdot}\leq\lambda_1}(f) +
    \lambda_2\norm[\radon]{q} - \int_\Omega q\cdot \grad
    f.\label{eq:saddle-point-dual-form}
  \end{split}
\end{equation}
This is of the form
\[
\min_{u,q}\max_{f,\phi} G(u,q) + \scp{K(u,q)}{(f,\phi)} - F(f,\phi)
\]
with
\begin{align*}
  G(u,q) &= \norm[\radon]{q}\\
  F(f,\phi) & = I_{\norm[\infty]{\cdot}\leq 1}(\abs{\phi}) +
  I_{\norm[\infty]{\cdot}\leq\lambda_1}(f)+\int_\Omega f\,u^0\d x\\
  K
  \begin{bmatrix}
    u\\q
  \end{bmatrix}
  & =
  \begin{bmatrix}
    \id & \Div\\
    \grad & 0
  \end{bmatrix}
  \begin{bmatrix}
    u\\q
  \end{bmatrix}
  =
  \begin{bmatrix}
    u + \Div q\\
    \grad u
  \end{bmatrix}
\end{align*}

\begin{remark}
  We may also start from the cascading
  formulation~\eqref{eq:krtv-cascade} which is already almost in
  saddle-point form:
  \begin{align}
    &\min_{u,\nu}\lambda_1\norm[\radon]{u-u_0 - \Div \nu} +
    \lambda_2\norm[\radon]{\abs{\nu}} + \TV(u)\nonumber\\
    = & \min_{u,\nu}\max_\phi\lambda_1\norm[\radon]{u-u_0 - \Div \nu}
    +
    \lambda_2\norm[\radon]{\abs{\nu}} + \int_\Omega\grad u\cdot\phi\d x - I_{\norm[\infty]{\cdot}\leq 1}(\abs{\phi})\nonumber\\
    = & \min_{u,\nu}\max_{f,\phi}\int_\Omega(u-u_0 - \Div \nu)\,f \d x
    + \lambda_2\norm[\radon]{\abs{\nu}} + \int_\Omega\grad
    u\cdot\phi\d x\nonumber\\
    &\qquad\qquad- I_{\norm[\infty]{\cdot}\leq 1}(\abs{\phi}) -
    I_{\norm[\infty]{\cdot}\leq\lambda_1}(f).\nonumber\label{eq:saddlepoint-cascade}
  \end{align}
  However, using $-\int_\Omega \Div\nu \, f\d x = \int_\Omega
  \nu\cdot\grad f\d x$ we arrive back at precisely the same
  formulation as~(\ref{eq:saddle-point-dual-form}) (with $\nu$ instead
  of $q$).
\end{remark}

Note that both $F$ and $G$ admit simple proximity operators (both
implementable in complexity proportional to the number of variables in
$F$ or $G$, respectively). Moreover, the operators $K$ and its adjoint
involve only one application of the gradient and the divergence (and
some pointwise operations) and hence, can also be implemented in
linear complexity. Hence, the application of general first order
primal dual methods leads to methods with very low complexity of the
iterations and usually fast initial progress of the
iterations. Moreover, note that the norm of $K$ can be estimated with
the help of the norm of the (discretized) gradient operator as
$\norm{K}\leq \sqrt{\norm{\grad}^2 + 2}$. In our experiments we used
the inertial forward-backward primal-dual method
from~\cite{lorenz2014accelerated} with a constant inertial parameter
$\alpha$.

For our one-dimensional examples in Section~\ref{sec:one-dim-ex} the
total number of variables is small enough so that general purpose
solvers for convex optimization can be applied. Here we used
CVX~\cite{gb08,cvx} with the interior point solver from
MOSEK.\footnote{\url{http://mosek.com}}

\section{Experiments}
\label{sec:experiments}

In this section we present examples of minimizers of the $\KR$-$\TV$
problem. In each subsection we do not have the aim to show that
$\KR$-$\TV$ outperforms any existing method but to point out
additional features of this new approach. Hence, we do in general not
compare the $\KR$-$\TV$ functional against the most successful method
for the respective task, but to the closest relative among the
successful methods, i.e. to the $L^1$-$\TV$ method.

\subsection{One-dimensional examples}
\label{sec:one-dim-ex}

Figure~\ref{fig:1d-KT-TV} shows the influence of the parameters
$\lambda_1$ and $\lambda_2$ in three simple but instructive examples:
a plateau, a ramp and a hat.

For the $L^1$-$\TV$ case the plateau either stays exact (for $\lambda_1$
large enough) or totally disappears (for $\lambda_1$ small enough).
If the plateau would have been wide enough, the it would not disappear
but the minimizer would be constant 1 since the minimizer always
approaches the constant median value for $\lambda_1\to 0$. In the
$\KR$-$\TV$ case, however, the plateau gets wider and flatter while the
total mass is preserved. In the limit $\lambda_2\to 0$ the minimizer
converges to a constant but still has the same mass than $u^0$ since
for $\lambda_2\to 0$ one approaches the constant mean value.

For the ramp, $L^1$-$\TV$ shows the known behavior that the ramp is
getting flatter and flatter for decreasing $\lambda_1$. In the limit
$\lambda_1\to 0$ one obtains the constant median. For $\KR$-$\TV$,
somewhat unexpectedly, the ramp not only gets flatter (it approaches the
constant mean value, which equals the median here) but also forms new
jumps. For some parameter value, the minimizer is even a pure jump.

The observation for the hat is somehow similar to the ramp: $L^1$-$\TV$
just cuts off the hat-tip while $\KR$-$\TV$ creates additional jumps.

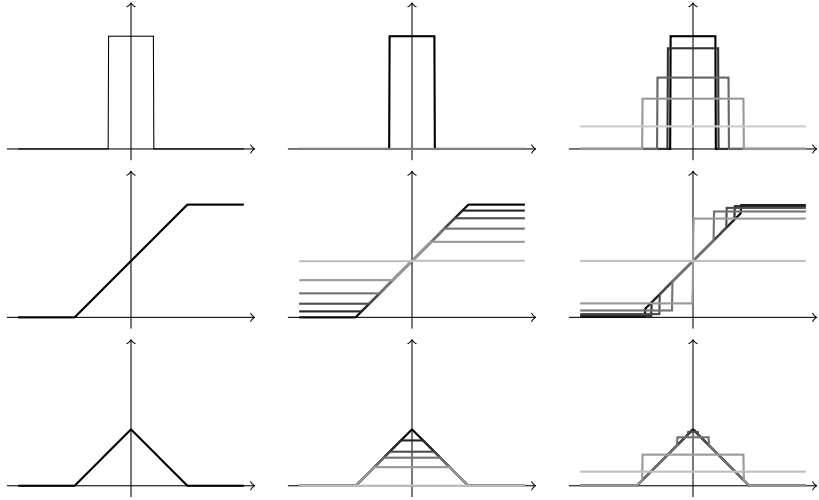
\begin{figure}[htb]
  \centering
  \begin{tabular}{ccc}
    \begin{tikzpicture}[scale=1.5]
      \draw[->] (-1.1,0) -- (1.1,0);
      \draw[->] (0,-0.1) -- (0,1.3);
      \draw plot file {data/varying_lambda1_lambda2_1d/plateau.dat};
      
    \end{tikzpicture}
    & 
    \begin{tikzpicture}[scale=1.5]
      \draw[->] (-1.1,0) -- (1.1,0);
      \draw[->] (0,-0.1) -- (0,1.3);
      \draw[thick,black!100] plot file {data/varying_lambda1_lambda2_1d/plateau_lambda1_10_lambda2_1000.dat};
      \draw[thick,black!50] plot file {data/varying_lambda1_lambda2_1d/plateau_lambda1_2_lambda2_1000.dat};
    \end{tikzpicture}
    &
    \begin{tikzpicture}[scale=1.5]
      \draw[->] (-1.1,0) -- (1.1,0);
      \draw[->] (0,-0.1) -- (0,1.3);
      \draw[thick,black!100] plot file {data/varying_lambda1_lambda2_1d/plateau_lambda1_100_lambda2_80.dat};
      \draw[thick,black!80] plot file {data/varying_lambda1_lambda2_1d/plateau_lambda1_100_lambda2_40.dat};
      \draw[thick,black!60] plot file {data/varying_lambda1_lambda2_1d/plateau_lambda1_100_lambda2_20.dat};
      \draw[thick,black!40] plot file {data/varying_lambda1_lambda2_1d/plateau_lambda1_100_lambda2_10.dat};
      \draw[thick,black!20] plot file {data/varying_lambda1_lambda2_1d/plateau_lambda1_100_lambda2_5.dat};
    \end{tikzpicture}
    \\
    \begin{tikzpicture}[scale=1.5]
      \draw[->] (-1.1,0) -- (1.1,0);
      \draw[->] (0,-0.1) -- (0,1.3);
      \draw[thick] plot file {data/varying_lambda1_lambda2_1d/ramp.dat};
    \end{tikzpicture}
    & 
    \begin{tikzpicture}[scale=1.5]
      \draw[->] (-1.1,0) -- (1.1,0);
      \draw[->] (0,-0.1) -- (0,1.3);
      \draw[thick,black!100] plot file {data/varying_lambda1_lambda2_1d/ramp_lambda1_2_lambda2_1000.dat};
      \draw[thick,black!85] plot file {data/varying_lambda1_lambda2_1d/ramp_lambda1_1.8_lambda2_1000.dat};
      \draw[thick,black!70] plot file {data/varying_lambda1_lambda2_1d/ramp_lambda1_1.6_lambda2_1000.dat};
      \draw[thick,black!55] plot file {data/varying_lambda1_lambda2_1d/ramp_lambda1_1.4_lambda2_1000.dat};
      \draw[thick,black!40] plot file {data/varying_lambda1_lambda2_1d/ramp_lambda1_1.2_lambda2_1000.dat};
      \draw[thick,black!25] plot file {data/varying_lambda1_lambda2_1d/ramp_lambda1_1_lambda2_1000.dat};
    \end{tikzpicture}
    &
    \begin{tikzpicture}[scale=1.5]
      \draw[->] (-1.1,0) -- (1.1,0);
      \draw[->] (0,-0.1) -- (0,1.3);
      \draw[thick,black!100] plot file {data/varying_lambda1_lambda2_1d/ramp_lambda1_100_lambda2_6.dat};
      \draw[thick,black!85] plot file {data/varying_lambda1_lambda2_1d/ramp_lambda1_100_lambda2_5.dat};
      \draw[thick,black!70] plot file {data/varying_lambda1_lambda2_1d/ramp_lambda1_100_lambda2_4.dat};
      \draw[thick,black!55] plot file {data/varying_lambda1_lambda2_1d/ramp_lambda1_100_lambda2_3.dat};
      \draw[thick,black!40] plot file {data/varying_lambda1_lambda2_1d/ramp_lambda1_100_lambda2_2.dat};
      \draw[thick,black!25] plot file {data/varying_lambda1_lambda2_1d/ramp_lambda1_100_lambda2_1.dat};
    \end{tikzpicture}
    \\
    \begin{tikzpicture}[scale=1.5]
      \draw[->] (-1.1,0) -- (1.1,0);
      \draw[->] (0,-0.1) -- (0,1.3);
      \draw[thick] plot file {data/varying_lambda1_lambda2_1d/hat.dat};
      
    \end{tikzpicture}
    & 
    \begin{tikzpicture}[scale=1.5]
      \draw[->] (-1.1,0) -- (1.1,0);
      \draw[->] (0,-0.1) -- (0,1.3);
      \draw[thick,black!100] plot file {data/varying_lambda1_lambda2_1d/hat_lambda1_300_lambda2_100000.dat};
      \draw[thick,black!85] plot file {data/varying_lambda1_lambda2_1d/hat_lambda1_10_lambda2_100000.dat};
      \draw[thick,black!70] plot file {data/varying_lambda1_lambda2_1d/hat_lambda1_5_lambda2_100000.dat};
      \draw[thick,black!55] plot file {data/varying_lambda1_lambda2_1d/hat_lambda1_4_lambda2_100000.dat};
      \draw[thick,black!40] plot file {data/varying_lambda1_lambda2_1d/hat_lambda1_3_lambda2_100000.dat};
      \draw[thick,black!25] plot file {data/varying_lambda1_lambda2_1d/hat_lambda1_2_lambda2_100000.dat};
    \end{tikzpicture}
    &
    \begin{tikzpicture}[scale=1.5]
      \draw[->] (-1.1,0) -- (1.1,0);
      \draw[->] (0,-0.1) -- (0,1.3);
      \draw[thick,black!100] plot file {data/varying_lambda1_lambda2_1d/hat_lambda1_300_lambda2_100000.dat};
      \draw[thick,black!85] plot file {data/varying_lambda1_lambda2_1d/hat_lambda1_300_lambda2_10000.dat};
      \draw[thick,black!70] plot file {data/varying_lambda1_lambda2_1d/hat_lambda1_300_lambda2_1000.dat};
      \draw[thick,black!55] plot file {data/varying_lambda1_lambda2_1d/hat_lambda1_300_lambda2_100.dat};
      \draw[thick,black!40] plot file {data/varying_lambda1_lambda2_1d/hat_lambda1_300_lambda2_10.dat};
      \draw[thick,black!25] plot file {data/varying_lambda1_lambda2_1d/hat_lambda1_300_lambda2_1.dat};
    \end{tikzpicture}
  \end{tabular}
  \caption{One-dimensional illustrations for KR-TV denoising with
    varying parameters. Left: Original functions $u^0$. Middle:
    Corresponding $L^1$-$\TV$ minimizers with $\lambda_1$ decreasing
    (lighter gray corresponds to smaller $\lambda_1$); $\lambda_2$ is
    so large, that the respective constraint is inactive
    throughout. Right: Corresponding $\KR$-$\TV$ minimizers with
    decreasing $\lambda_2$ (lighter gray corresponds to smaller
    $\lambda_2$); $\lambda_1$ is so large, that the respective
    constraint is inactive throughout.}
  \label{fig:1d-KT-TV}
\end{figure}

\subsection{Two dimensional denoising with $\KR$-$\TV$}
\label{sec:two-dim-KR-TV-denoise}

We illustrate the denoising capabilities of $\KR$-$\TV$ in comparison
with $L^1$-$\TV$ in Figures~\ref{fig:2d-KR-TV_denoise_parameters}
and~\ref{fig:2d-KR-TV-denoise}.
Figure~\ref{fig:2d-KR-TV_denoise_parameters} shows effects similar to
those shown in~Figure~\ref{fig:1d-KT-TV} in one dimension. While both
$L^1$-$\TV$ and $\KR$-$\TV$ denoise the image well, $L^1$-$\TV$ tends to
remove small structures completely while $\KR$-$\TV$ mashes small
structures together before they are merged with the background.

\begin{figure}[htb]
  \centering
  \begin{tabular}{ccc}
    &
    $L^1$-$\TV$
    & 
    $\KR$-$\TV$
    \\
    \includegraphics[width=0.3\textwidth]{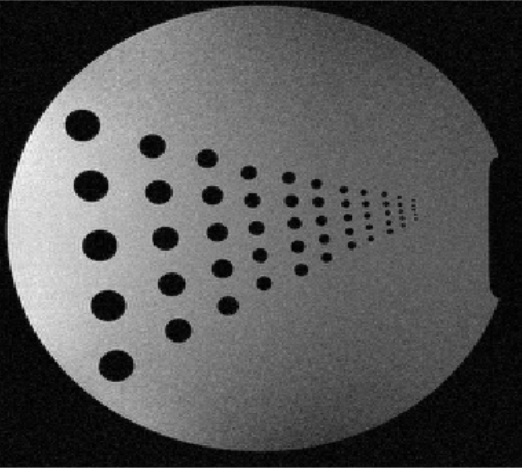} &
    \includegraphics[width=0.3\textwidth]{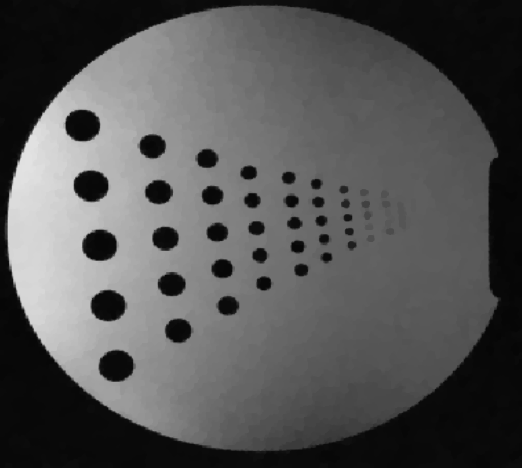} &
    \includegraphics[width=0.3\textwidth]{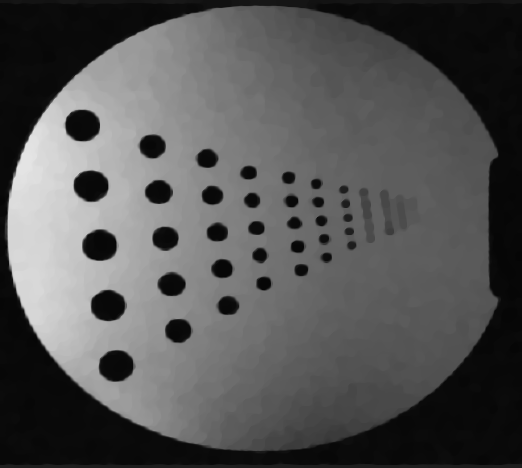}\\
    &
    $\lambda_1= 0.6$
    &
    $\lambda_2 = 0.004$
    \\
    &
    \includegraphics[width=0.3\textwidth]{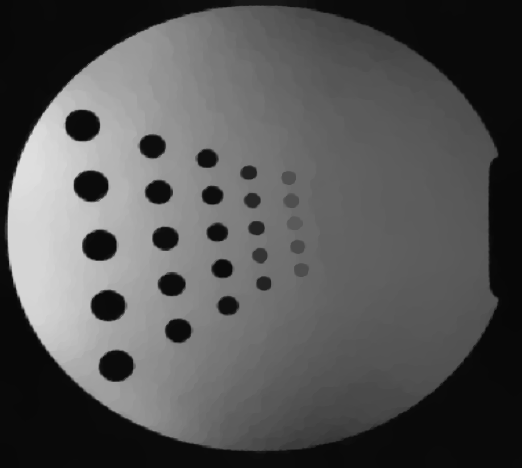} &
    \includegraphics[width=0.3\textwidth]{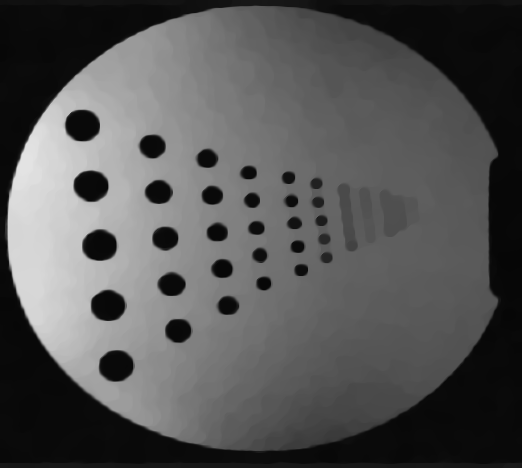} \\
    &
    $\lambda_1= 0.3$
    &
    $\lambda_2 = 0.002$
  \end{tabular}   
  \caption{Denoising with $\KR$-$\TV$ and $L^1$-$\TV$. In the right
    images $\lambda_1$ is so large that the respective constraint is
    inactive.}
  \label{fig:2d-KR-TV_denoise_parameters}
\end{figure}

In Figure~\ref{fig:2d-KR-TV-denoise} we took a piecewise affine image,
contaminated by noise and denoised it by $L^1$-$\TV$ and $\KR$-$\TV$. The
parameters $\lambda_1$, respectively $\lambda_2$ have been tuned by
hand to give a minimal $L^1$-error to the ground truth, i.e. to the noise-free $u^\dag$. Even though this choice seem to be perfectly suited
for $L^1$-$\TV$ it turns out that $\KR$-$\TV$ achieves a smaller
error. Also note that staircasing is slightly reduced but also edges
are a little more blurred for $\KR$-$\TV$ than for $L^1$-$\TV$.

\begin{figure}[htb]
  \centering

  \begin{tabular}{ccc}
    &
    \includegraphics[width=0.3\textwidth]{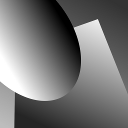}
    &
    \\
    &
    $u^\dag$
    &
    \\
    \includegraphics[width=0.3\textwidth]{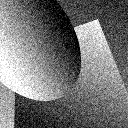}
    &
    \includegraphics[width=0.3\textwidth]{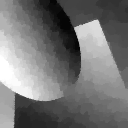}
    &
    \includegraphics[width=0.3\textwidth]{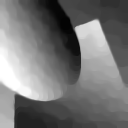}\\
    noisy, $u^0$
    &
    $L^1$-TV
    &
    KR-TV \\
    & 
    $\norm[L^1]{u-u^\dag} = 295.7$
    & 
    $\norm[L^1]{u-u^\dag} = 253.7$
  \end{tabular}
  \caption{Denoising with $\KR$-$\TV$ and $L^1$-$\TV$. Left: Noisy image,
    middle: KR-TV denoised by using the value $\lambda_2$ only
    ($\lambda_1$ so large, that the bound is inactive), right,
    $L^1$-TV denoising (i.e. only $\lambda_1$ is used). The respective
    values $\lambda_1$ and $\lambda_2$ have been optimized to result
    is the smallest $L^1$ error to the original noise-free image.}
  \label{fig:2d-KR-TV-denoise}
\end{figure}

\subsection{Cartoon-Texture decomposition}
\label{sec:cartoon-texture}

We compare the $\KR$-$\TV$ model for cartoon texture decomposition with $L^1$-$\TV$ and also with Meyer's $G$-$\TV$ (cf.~Section~\ref{sec:re-g-norm}).
In Figure~\ref{fig:2d-KR-TV-cartoon-texture} we show decompositions of
Barbara into its cartoon and texture part.  
The parameters have been chosen as follows: We started with the value $\lambda_1$ for the $L^1$-$\TV$ decomposition (i.e. $\lambda_2=\infty$) and chose it such that
 most texture is in the texture component but
also some structure is already visible. Then, for the $G$-$\TV$ the parameter was adjusted such that the cartoon part has the same total variation as the cartoon part from the $L^1$-$\TV$ decomposition. For the $\KR$-$\TV$ decomposition, the value $\lambda_1$ was set to $\infty$
      while $\lambda_2$ was again chosen such that the total variation of the cartoon part equals the total variation of the other cartoon parts.
The rationale behind this choice is, that the total variation is used as a prior for the cartoon part in all three models.
We remark that the choosing the parameters such that the $L^1$-discrepancy of the texture part is equal for all three decompositions leads to slightly different, but visually comparable results.

 Note that, for these parameters 
the $L^1$-$\TV$ decomposition already has
some structure in the texture part (parts of the face and of the
bookshelf) and the $G$-$\TV$ decomposition has structure and texture severely mixed,  while for $\KR$-$\TV$ the texture component still mainly
contains texture. Also note that $\KR$-$\TV$ manages to keep the
smooth structure of the clothes in the cartoon part (see e.g. the
scarf and the trousers) while $L^1$-$\TV$ gives a more ``constant''
cartoon image.

\begin{figure}[htb]
  \centering
  \begin{tabular}{cccc}
    Original
    &
    $L^1$-$\TV$
    &
    $G$-$\TV$
    &
    $\KR$-$\TV$\\
    \begin{minipage}{0.23\linewidth}
      \includegraphics[width=\textwidth]{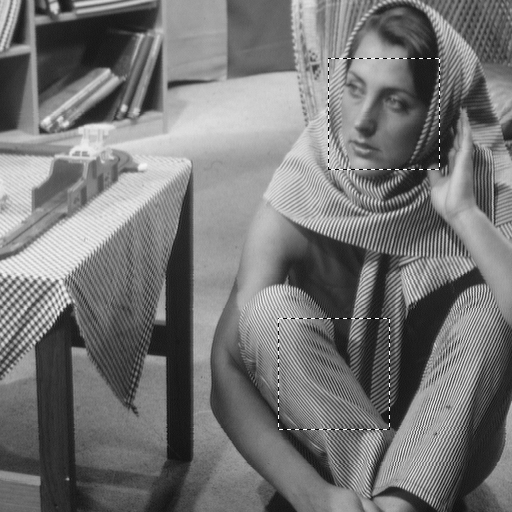}\\
      \includegraphics[width=0.48\textwidth]{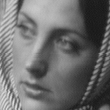}\hfill
      \includegraphics[width=0.48\textwidth]{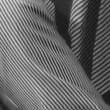}
    \end{minipage}
    & 
    \begin{minipage}{0.23\linewidth}
      \includegraphics[width=\textwidth]{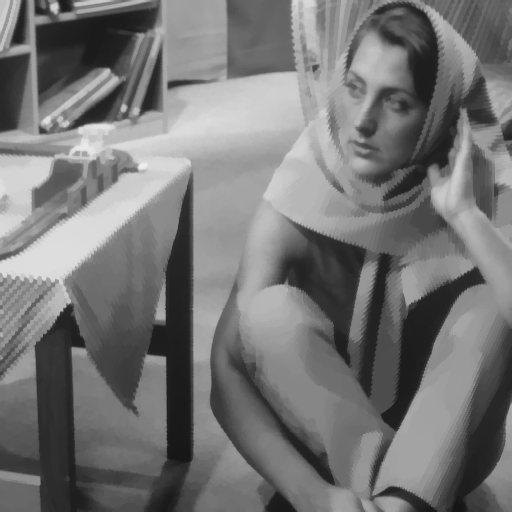}\\
      \includegraphics[width=0.48\textwidth]{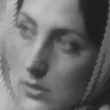}\hfill
      \includegraphics[width=0.48\textwidth]{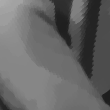}
    \end{minipage}
    &
    \begin{minipage}{0.23\linewidth}
      \includegraphics[width=\textwidth]{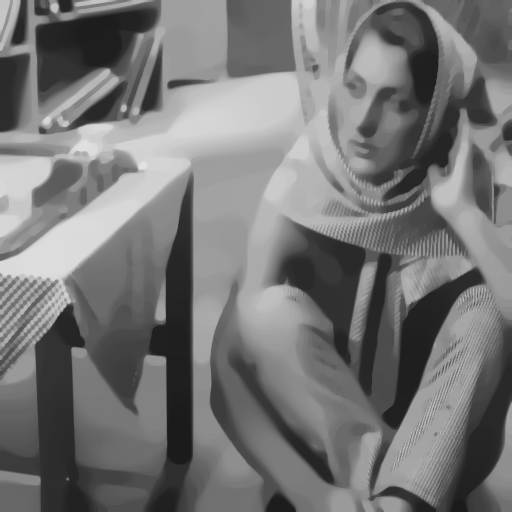}\\
      \includegraphics[width=0.48\textwidth]{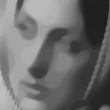}\hfill
      \includegraphics[width=0.48\textwidth]{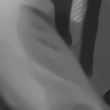}
    \end{minipage}
    & 
    \begin{minipage}{0.23\linewidth}
      \includegraphics[width=\textwidth]{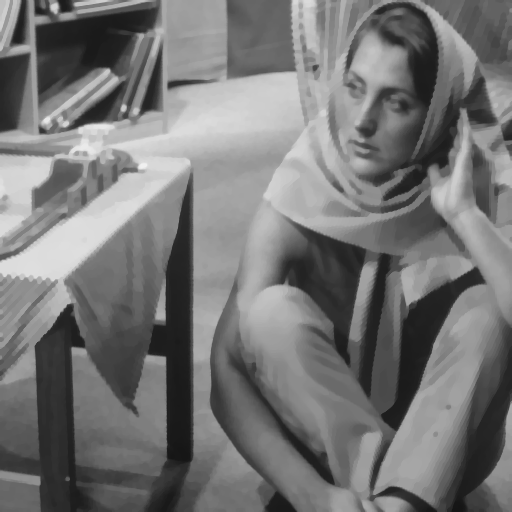}\\
      \includegraphics[width=0.48\textwidth]{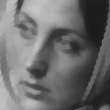}\hfill
      \includegraphics[width=0.48\textwidth]{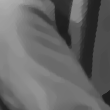}
    \end{minipage}\vspace*{\medskipamount}\\
    & 
    \begin{minipage}{0.23\linewidth}
      \includegraphics[width=\textwidth]{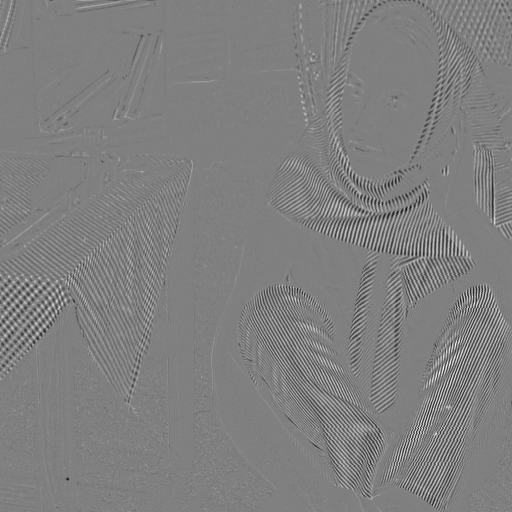}\\
      \includegraphics[width=0.48\textwidth]{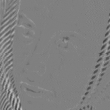}\hfill
      \includegraphics[width=0.48\textwidth]{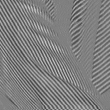}
    \end{minipage}
    &
    \begin{minipage}{0.23\linewidth}
      \includegraphics[width=\textwidth]{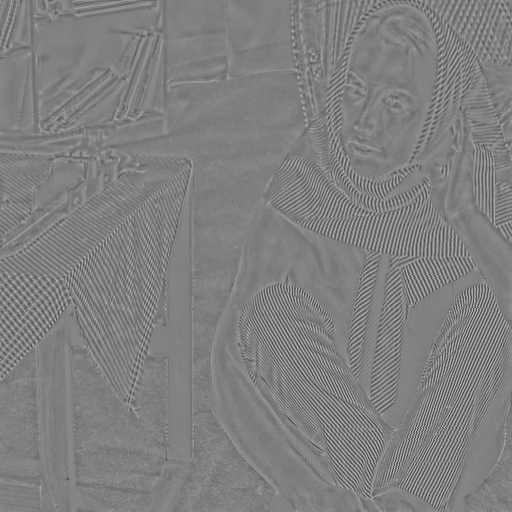}\\
      \includegraphics[width=0.48\textwidth]{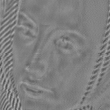}\hfill
      \includegraphics[width=0.48\textwidth]{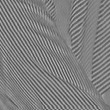}
    \end{minipage}
    & 
    \begin{minipage}{0.23\linewidth}
      \includegraphics[width=\textwidth]{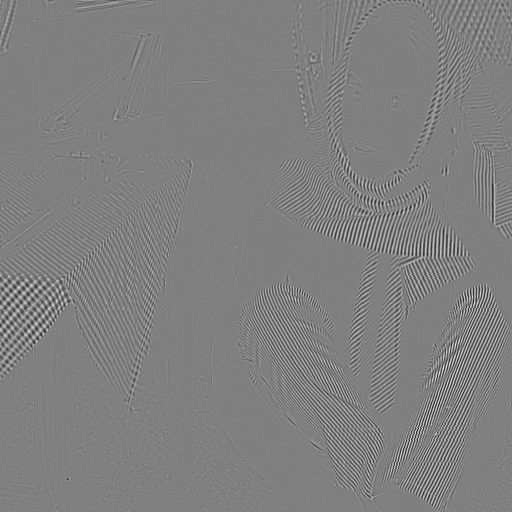}\\
      \includegraphics[width=0.48\textwidth]{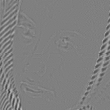}\hfill
      \includegraphics[width=0.48\textwidth]{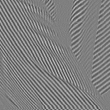}
    \end{minipage}
    \end{tabular}
    \caption{Cartoon-texture decomposition with $L^1$-$\TV$, $G$-$\TV$, and $\KR$-$\TV$. Top row: original and cartoon parts, bottom row: texture parts.}
  \label{fig:2d-KR-TV-cartoon-texture}
\end{figure}

\section{Conclusion}
\label{sec:conclusion}
In this paper we propose a new discrepancy term in a total variation
regularisation approach for images that is motivated by optimal
transport. The proposed discrepancy term is the Kantorovich-Rubinstein
transport norm. We show relations of this norm to other standard
discrepancy terms in the imaging literature and derive qualitative
properties of minimizers of a total variation regularization model
with a KR discrepancy. Indeed, we find that the KR discrepancy can be
seen as a generalization of the dual Lipschitz norm and the $L^1$ norm,
both of which can be derived from the Kantorovich-Rubinstein norm by letting one of the
parameters go to infinity, respectively. Moreover, we show that this
specialization is in fact crucial for obtaining a model in which the
solution conserves mass and that the model has a solution which
preserves positivity.

The paper is furnished with a discussion of experiments where we use
the $\KR$-$\TV$ regularisation approach in the context of image
denoising and image decomposition. Our numerical discussion suggests
that the use of the $\KR$ norm can reduce the $\TV$ staircasing effect
and performs better when decomposing an image into a cartoon-like and
oscillatory component. Due to the mass conversation property we also
expect that this approach is interesting in medical imaging, where
images are usually indeed density functions of physical quantities, as
well as in the context of density estimation where total variation
approaches have been used before in the context of earthquakes and
fires, see~\cite{mohler2011fast} for instance.  The applicability of
the $\KR$ discrepancy in other imaging problems such as optical flow,
image sequence interpolation or stereo vision has to be investigated
in future research.

While some analytical properties of the $\KR$-$\TV$ method have been
established (e.g. a weak maximum principle and a mass preservation
property), a deeper understanding of the geometrical properties, as
has been carried out for and $L^1$-$\TV$ and $L^2$-$\TV$, as well
as for $\TGV$ on one-dimensional domains (see,
e.g.,~\cite{strong2003edgepreserving,caselles2007discontinuity,duval2009geometricl1tv,l1tgv,papafitsoros2013onedimtgv,poeschl2013exacttgv}),
would indeed be interesting. However, due to the non-locality of the
$\KR$ discrepancy, the analysis may be more complicated.

\section*{Acknowledgement}

This project has been financially supported by the King Abdullah
University of Science and Technology (KAUST) Award No.~KUK-I1-007-43,
and the EPSRC first grant Nr.~EP/J009539/1 ``Sparse \& Higher-order 
Image Restoration''. T.~Valkonen has further been supported by a 
Senescyt (Ecuadorian ministry of Education, Science, and Technology)
Prometeo Fellowship. J.~Lellmann has been supported by the Leverhulme
Early Career Fellowship ECF-2013-436.

\bibliographystyle{plain}
\bibliography{references}
\end{document}